\documentclass[runningheads]{llncs}

\usepackage{eccv}

\usepackage{eccvabbrv}

\usepackage{dsfont}
\usepackage{multirow}
\usepackage{multicol}
\usepackage{subcaption}
\usepackage{floatrow}
\usepackage{pifont}
\usepackage{stfloats}
\usepackage{pgfplots}
\usepackage{pgfplotstable}
\pgfplotsset{compat=1.18} 
\usepackage{colortbl}
\definecolor{Gray}{gray}{0.9}
\definecolor{vitred}{RGB}{214,39,40}
\definecolor{posblue}{RGB}{31,119,180}
\definecolor{neggreen}{RGB}{44,160,44}

\newcolumntype{g}{>{\columncolor{Gray}}c}
\usepackage{layout}
\usepackage{amsfonts}
\usepackage{amsmath}
\usepackage{amssymb} % NOT COMATIBLE WITH svjour3

\usepackage{mdframed} 
\usepackage{thmtools}
\usepackage{soul}

\definecolor{shadecolor}{gray}{0.95}
\declaretheoremstyle[
headfont=\normalfont\bfseries,
notefont=\mdseries, notebraces={(}{)},
bodyfont=\normalfont,
postheadspace=0.5em,
spaceabove=1pt,
mdframed={
  skipabove=8pt,
  skipbelow=8pt,
  hidealllines=true,
  backgroundcolor={shadecolor},
  innerleftmargin=4pt,
  innerrightmargin=4pt}
]{shaded}

\usepackage{xcolor}
\usepackage{color}
\usepackage{verbatim}

\newcommand{\R}{\mathbb{R}} % Reals
\newcommand{\cA}{{\cal A}}

\newcommand{\cC}{{\cal C}}

\newcommand{\cL}{{\cal L}}

\newcommand{\cN}{{\cal N}}

\usepackage{mdframed} 
\newcommand{\myNum}[1]{(\emph{#1})}
\newcommand{\smartparagraph}[1]{\vspace{2pt} \noindent {\bf #1}}

\usepackage{algorithm}
\usepackage{algpseudocode}
\usepackage{graphicx}
\usepackage{booktabs}
\usepackage{subcaption}
\usepackage{caption}

\definecolor{codeblue}{rgb}{0.21,0.49,0.74}
\definecolor{codegreen}{rgb}{0,0.6,0}
\definecolor{correctgreen}{HTML}{00CC00}
\newcommand{\greenup}{\textcolor{codegreen}{$\uparrow$}}
\newcommand{\reddown}{\textcolor{red}{$\downarrow$}}
\usepackage[accsupp]{axessibility} 

\expandafter\def\expandafter\normalsize\expandafter{%
    \normalsize%
    \setlength\abovedisplayskip{0pt}%
    \setlength\belowdisplayskip{4pt}%
    \setlength\abovedisplayshortskip{-1.2pt}%
    \setlength\belowdisplayshortskip{1pt}%
}

\usepackage{hyperref}

\definecolor{Gray}{gray}{0.9}

\newcommand{\attnname}{DnA\xspace}
\newcommand{\attnsigma}{DnA$^\sigma$\xspace}

\usepackage{orcidlink}

\begin{document}

\title{DnA: Denoising Attention for Visual Tasks}

\titlerunning{Denoising Attention}

\author{Ron Campos\inst{1}\orcidlink{0009-0001-8712-6727},
Subhajit Maity \inst{1}\orcidlink{0000-0002-0735-8406}, 
Xin Li \inst{1}\orcidlink{0000-0003-1201-9131},
Srijan Das\inst{2}\orcidlink{0000-0002-3373-6749},
Aritra Dutta\inst{1}\orcidlink{0000-0001-6994-1659}}

\authorrunning{R. Campos et al.}

\institute{University of Central Florida, USA \and
University of North Carolina at Charlotte, USA\\
\textbf{Project Page:} \url{https://rjccv.github.io/DnA}}

\renewcommand{\footnoterule}{%
  \kern -3pt
  \hrule width 1.0\columnwidth height 0.4pt
  \kern 2.6pt
}

\maketitle
\enlargethispage{\baselineskip}
\begingroup
\renewcommand\thefootnote{}
\footnotetext{\scriptsize Accepted at the 19\textsuperscript{th} European Conference on Computer Vision (ECCV), 2026.}
\endgroup

\begin{abstract}
The softmax activation in multihead attention (MHA) is the de facto standard for attention-based models in visual perception tasks. However, standard softmax can produce noisy attention patterns that dilute relevant features and degrade its performance. In this paper, we propose \textbf{D}e\textbf{n}oising \textbf{A}ttention or \textbf{\attnname}, in which, first, a positive query identifies which image features belong to the correct class, and a negative query identifies closely associated but irrelevant image features. \attnname then projects these interactions into two distinct subspaces with larger principal angles, promoting subspace separation and improved discriminability. Using a ViT-B backbone, our proposed \attnname achieves an absolute gain of 0.8$\%$ on ImageNet-1K compared to the baseline. We further show improvements across multiple visual understanding tasks, including video understanding with video transformers (1.8\%) and video LLMs (0.5\%). Our extensive empirical analyses justify the design choices involving two interacting subspaces and the denoising effect of \attnname. The code is publicly available at \url{https://github.com/rjccv/DnA}.
\end{abstract}

\begin{figure}[t]
    \centering
    \includegraphics[width=\linewidth]{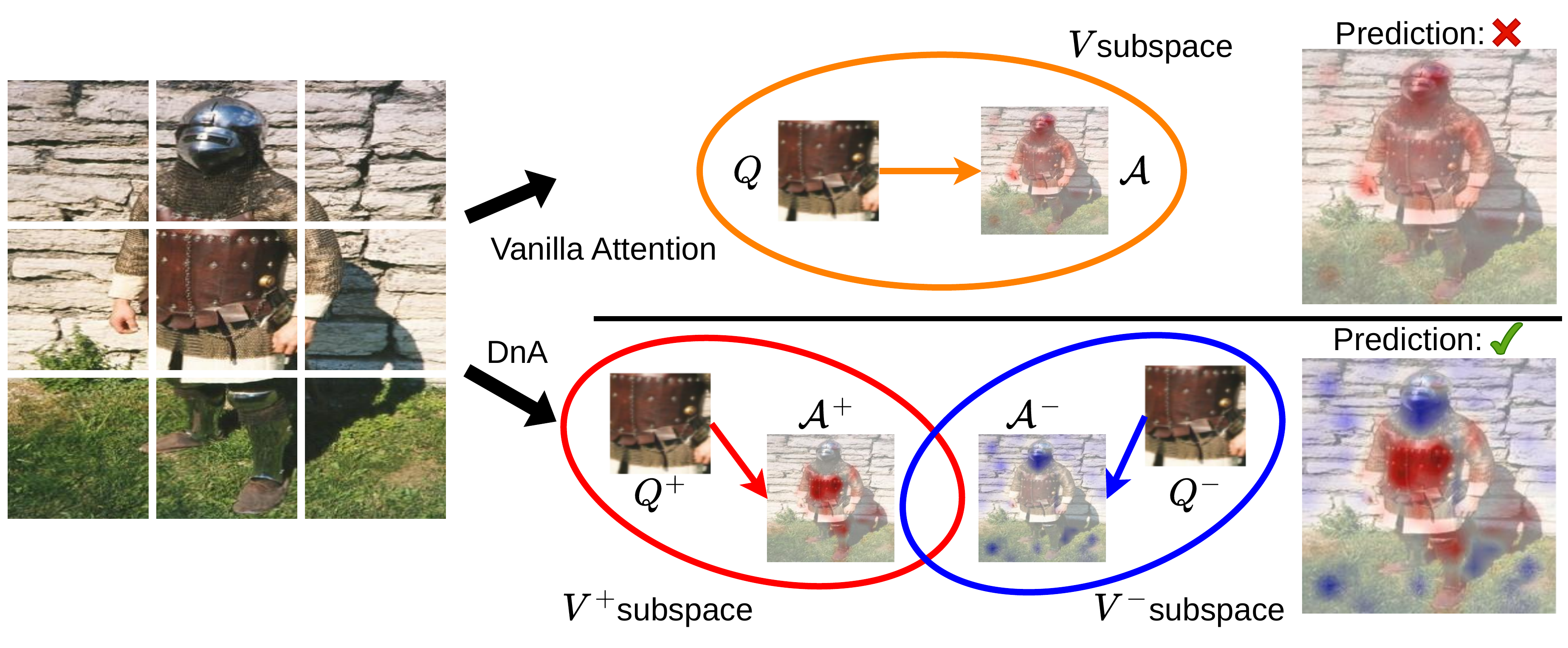}
    \caption{\small{The original image from ImageNet-1K \cite{deng2009imnet} shows the primary object, \textbf{breastplate}, along with secondary adversarial objects, \emph{person, helmet}, \etc. The traditional attention~\cite{vit, deit} and differential attention \cite{yedifferential} project the interactions onto a single $V$ subspace, resulting in misallocation of attention to secondary objects. Our proposed \attnname explicitly models \textcolor{red}{positive} and \textcolor{blue}{negative} interactions using two sets of queries, $Q^+$ and $Q^-$, and projects them onto two separate subspaces, $V^+$ and $V^-$, respectively, to effectively segregate closely associated adversarial features from the relevant ones. We define traditional attention as $\cA$, and DnA branches as $\cA^+$ and $\cA^-$.}}
    \label{fig:banner}
\end{figure}
\section{Introduction}\label{sec:intro}

The softmax function is a fundamental component in machine learning and statistics that transforms arbitrary real-valued scores into a probability distribution. It plays an essential role in the Transformer architecture, particularly in the self-attention mechanism that forms the core of modern large language models \cite{team2023gemini, gpt3, llama32vision, Llama3, campos2025gaea} and other sequence-to-sequence models~\cite{chen2023pali,raffel2020exploring,bart}. While efficiency-driven alternatives exist for specific use cases \cite{fibottention, yun2020n, shi2021sparsebert, zhang2021multi, zaheer2020big, guo2019star, hanbridging, nguyen2023probabilistic, lu2021soft}, softmax remains the gold standard for attention normalization due to its unique combination of theoretical properties and empirical performance \cite{vaswani2017attn, wang2020linformer, tay2020efficient, vit, fibottention}. The attention in Transformers computes contextualized representations by allowing each position in a sequence to attend to all positions, with the attention weights determined by the softmax function; see details in \S\ref{sec:mhsa}.

Although the softmax amplifies the dominant scores, it treats positive and negative values asymmetrically; the numerator amplifies the gaps between positive scores and compresses the gaps between negative scores. Indeed, we have the following observation:
{\it For $\delta>0$, we have $e^{a+\delta}-e^a=e^a(e^\delta -1)= e^a\delta +O(\delta^2)$.}
As a result, because softmax is used in transformers, larger key-query interactions receive much higher scores, while smaller interactions lose expressivity. This asymmetry may limit the representational capacity of standard attention.

Beyond score magnitudes, the geometry of feature subspaces also affects discriminability; larger principal angles between class subspaces reduce classification error \cite{subspaceangle}. This perspective motivates modeling positive and negative token interactions as a two-class problem, suggesting that subspace separation combined with both positive and negative score interactions can improve classifier performance. We present this in the context of the database query language \cite{sqlbook, accesspathsql}, where the primary object might be associated with other secondary object classes. The key (token of interest), query (token-associated request), and value (token retrieval) interactions in traditional softmax attention; see Figure \ref{fig:banner}. In contrast to differential attention \cite{yedifferential} for NLP tasks, which models the attention matrix using a single subspace $V$, we introduce two separate queries over the same keys. First, \myNum{i} \textcolor{red}{Positive} query: \textit{Which features belong to the correct class?} and \myNum{ii} \textcolor{blue}{Negative} query: \textit{Which features do not belong to the correct class, but are closely associated?} Their interactions are projected into distinct subspaces, $V^+$ and $V^-$, with larger principal angles between them, promoting subspace separation and improved discriminability; see Figure \ref{fig:banner}.

Taken together, in this paper, we propose \textbf{D}e\textbf{n}oising \textbf{A}ttention or \textbf{\attnname}, which replaces traditional attention in transformers; see \S\ref{sec:method}. Since ViTs are integral to modern foundational models~\cite{oquab2024dinov2, sam, clip}, we explore the efficacy of \attnname in traditional transformers and Video LLMs. By replacing the traditional attention, our proposed \attnname outperforms ViT-B on ImageNet-1K~\cite{deng2009imnet} and achieves 0.8\% higher test accuracy. \attnname also excels on video understanding tasks, achieving up to an absolute 4.0\% gain on Toyota Smarthome~\cite{das2019toyotasmarthome}, 1.2\% gain on NTU RGB+D 60~\cite{shahroudy2016ntu} with TimeSformer~\cite{timesformer}, and 0.5\% gain on Ego-in-Exo PerceptionsMCQ~\cite{reilly2025egoexo} with the video LLM, VideoLLaMA3~\cite{zhang2023videollama}; see \S\ref{sec:expt}.

\section{Background and Related Works}\label{sec:related}

For completeness, we start with a brief description of multihead attention.

\subsection{Multi-Head Attention (MHA) in Vision}\label{sec:mhsa}

Let $X \in \mathbb{R}^{C \times R \times W}$ be a $C$ channel image with resolution $R\times W$. The image is flattened into $N$ tokens in $\mathbb{R}^D$ by $p \times p$ patches such that, $N = \frac{RW}{p^2}$ and $D = p^2C.$ The input, $X_p \in \mathbb{R}^{N \times D}$, is a sequence of $N$ tokens of dimension $D$. Transformers process the input through a series of $L$ encoder layers. For simplicity, we assume the input $X_p$ is encoder-agnostic. We note that the following applies equally to decoder blocks; for brevity, we only refer to encoders.

In each encoder, \textit{attention}
projects an input sequence on the row-space of three learnable matrices $W_Q, W_K, W_V \in \mathbb{R}^{D \times D}$ to produce query, $Q = X_Q W_Q$, key, $K = X_{KV} W_K$, and value, $V = X_{KV} W_V$. In \textit{self-attention}, the queries, keys, and values are derived from the same input, such that $X_Q = X_{KV} =  X_p \in \mathbb{R}^{N \times D}$. In \textit{cross-attention}, the queries are derived from one input sequence $X_Q \in \mathbb{R}^{N \times D}$, while the keys and values come from a different input sequence $X_{KV} \in \mathbb{R}^{M \times D}$, $M$ is the number of tokens. Next, the query, key, and values are partitioned among $H$ heads; for each head, $h$, we have $Q_h \in \mathbb{R}^{N \times d}$ and $K_h, V_h \in \mathbb{R}^{M \times d},$ where $d=\tfrac{D}{H}.$ The notion of \textit{multi-head} attention facilitates parallel attention calculation in all $H$ heads. For any head, the \textit{scaled dot-product}, $\frac{Q_h K_h^\top}{\sqrt{d}} \in \mathbb{R}^{N \times M},$ is further processed through nonlinear softmax activation function acting row-wise, followed by projection on the row space of the value matrix $V$, such that $\cA_h =  \boldsymbol{\sigma}\left(\frac{Q_h K_h^\top}{\sqrt{d}}\right)V_h \in \mathbb{R}^{N \times d}$. Self-attention is a special case where $M = N$. Once attention is calculated for all the heads, they are concatenated back to $\mathbb{R}^{N\times D}$.

The matrix, $\frac{Q_hK_h^\top}{\sqrt{d}}$, for a head inside an encoder, refers to how different image tokens \emph{attend} to each other. Self-attention models interactions among tokens of the same sequence, while cross-attention models how tokens from the query sequence relate to tokens in the source sequence. The softmax activation, $\boldsymbol{\sigma}$, acts on each row, normalizes the key-query interactions by projecting them onto a probability simplex, and models the token-to-token relationships; see \Cref{fig:arch}(i).

\smartparagraph{Studies That Improve Upon Attention.}\label{sec:attention improvement} Numerous works improve attention mechanisms in NLP and vision to enhance feature representation. DeiT \cite{deit} uses token-based distillation, CaiT \cite{touvron2021Deep} separates class-attention in deeper layers, DaViT \cite{ding2022davit} adds channel-wise attention, and Swin Transformer \cite{liu2021swin} employs hierarchical shifted-window attention for scalability. ConViT \cite{dascoli2021convit} introduces convolutional inductive bias through gated positional self-attention, while SimA \cite{koohpayegani2024sima} replaces softmax with $\ell_1$ normalization while maintaining competitive ViT performance. See broader discussion in Supplementary~\S\ref{app:additional_rel_works}.

\subsection{Attention Mechanisms in Video Understanding}\label{sec:attention foundation}

Attention mechanisms are crucial to current video understanding. We group the related work into two categories: traditional video transformers and vision-language adapters for video LLMs.

\smartparagraph{Transformers for video understanding.}
ViViT~\cite{arnab2021vivit} adapts a ViT to video by extracting the spatiotemporal tokens from videos and proposes several different transformer architectures that decouple the spatial and temporal dimensions. Video Swin Transformer~\cite{liu2022videoswin} implements a 3D shifted window-based multi-head self-attention to the spatiotemporal space, reducing the computation by applying attention within local 3D windows. TimeSformer~\cite{timesformer} proposes a divided space-time attention mechanism that decouples the attention operations.

\smartparagraph{Attention-based adapters in video LLMs.} 
Q-Former~\cite{blip2} and Perceiver Resampler~\cite{perceiver} use learnable queries that cross-attend to frozen visual encoder features, distilling visual context into a fixed number of tokens. Q-Former is used in Video-LLaMA~\cite{zhang2023videollama}, and Perceiver Resampler~\cite{perceiver} has been successfully integrated as a bridge between visual and text contexts for LLMs~\cite{idefics2, apollo, minicpmv}. mPLUG-Owl3 \cite{mplugowl3} introduces hyper-attention, which parallelizes cross- and self-attention by reusing language queries to attend to visual features and adaptively fuse them into the text stream.
VisCoP \cite{viscop} enables domain adaptation without fine-tuning the frozen vision encoder by using visual probes that cross-attend to intermediate features and extract domain-specific information.
\section{Model Architecture}\label{sec:method}

\begin{figure*}[t]
    \centering
    \includegraphics[width=\linewidth]{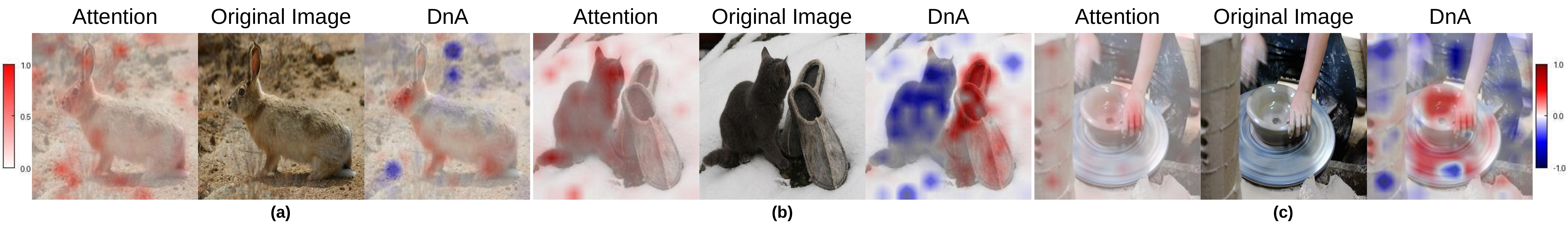}
    \caption{\small{
    Original images \textbf{(center)} from ImageNet-1K, their attention activation maps for traditional softmax \textit{only} showing \textcolor{red}{positive interactions} \textbf{(left)}, and our denoising attention, \attnname ($\cA_h^{Q\pm V\pm}$) \textbf{(right)} showing \textit{both} \textcolor{red}{positive} and \textcolor{blue}{negative interactions}. \textbf{(a)} The softmax attention is dispersed and does not focus on the concerned object (hare), while our \attnname focuses on the concerned object. \textbf{(b)} The softmax attention interacts heavily with a secondary adversarial class (cat), while our \attnname focuses on the concerned object (clog) and considers the adversarial object as misleading. \textbf{(c)} The softmax attention does not attend to the primary object (potter's wheel); only a few sporadic interactions are picked up from the background. In contrast, \attnname focuses on the primary object.}
    \label{fig:softmax_attn}
    }
\end{figure*}

This section introduces \textbf{D}e\textbf{n}oising \textbf{A}ttention or \textbf{\attnname}. We first analyze the softmax function, the core normalization mechanism in self-attention, and examine its properties from an information-theoretic perspective, highlighting potential limitations in representing negative token interactions that may be informative for visual understanding. Our focus is restricted to \textit{self-attention}; other components of the ViT architecture, including positional embedding, class token representation, pre- and post-normalization, MLP, \etc, remain untouched.

\subsection{Why Softmax for \emph{Attention}?}

For a vector, ${z}\in\R^n$ with components $(z_1, \cdots, z_n)$, the softmax function is defined component-wise, for $i\in [n]$, as $\boldsymbol{\sigma}({z})_i = \frac{e^{z_i}}{\sum_{j=1}^n e^{z_j}},$ which is a unique solution to several constrained optimization problems. One of the most fundamental characterizations of the softmax function is 
the {\it maximum entropy} property given in the Theorem below~\cite{jaynes1957information}.

\begin{theorem}[\emph{Maximum Entropy}]\label{thm:maxent}
Let ${z}=(z_1,\cdots ,z_n)$  
be a given vector of scores and let $Z$ be a random variable taking values $z_1,...,z_n$. The probability distribution ${p}^* = (p^*_1, \dots, p^*_n)$ that maximizes the Shannon entropy, $H({p}) = - \sum_{i=1}^n p_i \log p_i,$
subject to the constraints that ${p}$ is on the probability simplex, $\mathcal{P} = \{ {p} \in \mathbb{R}^n_{+} : \sum_{i=1}^n p_i = 1 \}$, and gives a fixed expected value $\mathbb{E}_{\mathbf{p}}[Z] = \mu$, for some $\mu \in \mathbb{R}$, is given by $p^*_i = \frac{e^{\lambda z_i}}{\sum_{j=1}^n e^{\lambda z_j}},$ where the parameter $\lambda \in \mathbb{R}$ is chosen such that $\sum_{i=1}^n p^*_i z_i = \mu$. 
\end{theorem}

The standard softmax is recovered when $\mu$ 
is the empirical average of the features, that is,  
$\mu=\sum_{i=1}^n z_ie^{z_i}/\sum_{j=1}^n e^{z_j}$.

The softmax distribution maximizes entropy subject to expected-value constraints, yielding the most noncommittal probability distribution consistent with the available information. This principle underlies its central role in information theory and statistical inference \cite{jaynes1957information, cover2006elements}. In attention mechanisms, softmax is therefore naturally justified as a principled method for transforming raw attention scores into a valid probability distribution \cite{vaswani2017attn, bahdanau2014neural}; see Theorem~\ref{thm:attention_probability} in \S\ref{app:theoretical_results}. This mechanism computes a weighted average, where weights encode the relative importance of each position. The exponential form of softmax induces nonlinear amplification of dominant scores, producing two effects:
\myNum{i} \textbf{Selective focus,} where large scores receive near-unit weight while small or negative scores are suppressed, emphasizing the most relevant information; and \myNum{ii} \textbf{Competitive normalization,} where increasing attention to one position necessarily decreases attention to others, enforcing a fixed attention budget.

\smartparagraph{Why do we need a \emph{softmin} function?} Because softmax pushes dominant scores toward 1 while compressing smaller scores toward 0, potentially informative low interactions may be suppressed. To capture such signals, we consider the softmin function, defined component-wise for $\mathbf z \in \mathbb{R}^n$ as $
\boldsymbol{\hat\sigma} ({\mathbf z})_i := \boldsymbol{\sigma} (-{\mathbf z})_i=\frac{e^{-z_i}}{\sum_{j=1}^n e^{-z_j}}.$ By Theorem~\ref{thm:maxent}, $\boldsymbol{\hat{\sigma}}$ is the unique distribution that maximizes Shannon entropy subject to the constraint $\sum_{j=1}^n \hat{p}_j (-z_j) = \mu$, equivalently $\sum_{j=1}^n \hat{p}_j z_j = -\mu$. Thus, softmin maximizes entropy under a mean constraint with the opposite sign, providing a principled counterpart to softmax that emphasizes strongly lower interactions. Thus, it makes sense to promote both the largest and the smallest components, since using only one provides incomplete information.

To conclude, we obtain a new principle: {\it When we do not have any information about which extreme value (largest or smallest) of the scores, we should use both softmax and softmin.}

\begin{figure*}[t]
    \centering
    \includegraphics[width=\linewidth]{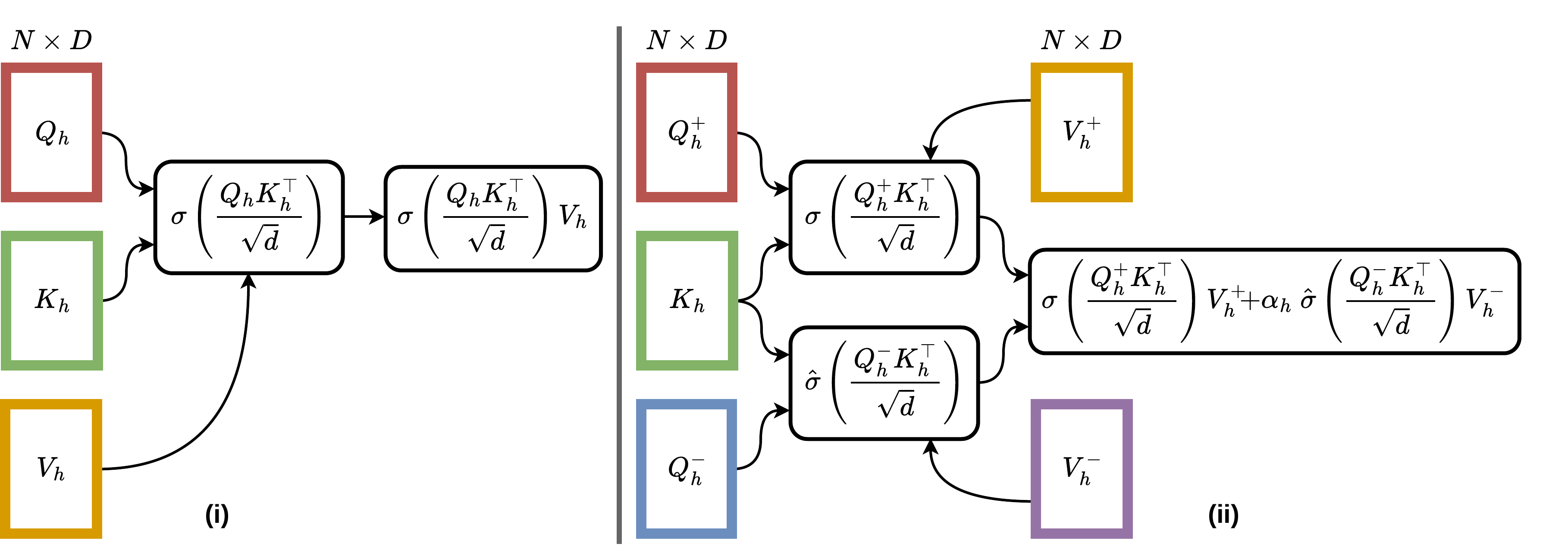}
    \caption{\small{A schematic comparison between traditional attention and the proposed denoising attention. \textit{(i)} The multi-head attention uses a softmax activation on the rows of the scaled dot-product $\tfrac{Q_hK_h^{\top}}{\sqrt{d}}$ and combines it with $V_h$. \textit{(ii)} The proposed denoising attention $\cA_h^{Q\pm V\pm}$, uses two sets of values, $V_h^{+}$ and $V_h^{-}$, and two sets of queries $Q_h^{+}$ and $Q_h^{-}$. We use \textit{softmax} and \textit{softmin} two scaled dot-products, $\tfrac{Q_h^{+}K_h^{\top}}{\sqrt{d}}$ and $\tfrac{Q_h^{-}K_h^{\top}}{\sqrt{d}}$, and are projected on the row-spaces of $V_h^{+}$ and $V_h^{-}$, respectively.}
    \label{fig:arch}
    }
\end{figure*}

\subsection{Designing {D}e{n}oising {A}ttention (\attnname)}
\label{sec:denoising_attention}

In ViTs~\cite{vit,deit}, self-attention models token interactions by aggregating features via normalized weights derived from the attention matrix $\cA_h$ (see~\Cref{fig:softmax_attn}). High interaction scores reflect strong mutual relevance and typically correspond to object-related features; see Figure~\ref{fig:softmax_attn}(a)-(b). However, as illustrated in~\Cref{fig:softmax_attn}(c), softmax attention may assign weight to irrelevant context, introducing noise or distraction, which we interpret as negative token interactions. Formally, softmax projects each row of the scaled dot-product $\frac{Q_h K_h^\top}{\sqrt{d}}$ onto the probability simplex, exponentially amplifying large positive scores while mapping low, negative, and highly negative scores near zero. Although weak interactions may be uninformative, strongly negative interactions can encode informative contrastive structure. Standard softmax suppresses such signals; we instead aim to explicitly model and leverage negative token interactions to enhance visual understanding.

Recently, differential attention~\cite{yedifferential} for language modeling is motivated by differential amplifiers, where subtracting two signals suppresses shared noise. The model projects the input $X$, to two sets of queries, ${Q_1, Q_2 \in \mathbb{R}^{N \times d/2}}$ and keys, ${K_1, K_2 \in \mathbb{R}^{N \times d/2}}$ and one value matrix ${V \in \mathbb{R}^{N \times d}}$, and the attention with a learnable parameter $\lambda$, is computed as $\cA_D(X) = (\boldsymbol{\sigma}(\tfrac{Q_1K_1^T}{\sqrt{d}}) - \lambda \boldsymbol{\sigma}(\tfrac{Q_2K_2^T}{\sqrt{d}}))V$. By construction, $\cA_D(X)$ lies in the row space, $C(V^\top)$ of $V$, meaning both interactions are projected onto the same subspace. This raises the question of whether improved modeling can be achieved by separating these projections.

To motivate this, we connect the performance of a subspace classifier and the principal angles between subspaces. Consider a two-class problem with labels, $C_1$ and $C_2$, concentrated near two subspaces with orthonormal bases, $U_1, U_2\in \R^{n\times d}.$ Let the class-conditional densities be modeled as:
\begin{equation*}
p(x|C_1) =\int p(x|\alpha, C_1)p(\alpha)d\alpha,\;{\rm and}\;p(x|C_2) =\int p(x|\alpha, C_2)q(\alpha)d\alpha,
\end{equation*}
where the vector $\alpha$ does not essentially follow a multivariate normal distribution. Let $P(C_2|C_1)$ be the probability of mistaking $C_2$ for $C_1$, and $P(C_1|C_2)$ be the opposite. Then the \emph{classification error}, $P_e = \tfrac{1}{2}\left(P(C_2|C_1)+P(C_1|C_2)\right).$ 
Following Bayes rule, one can rewrite $P(C_2|C_1)$ as
\begin{equation*}
P(C_2|C_1) =\int P(C_2|C_1, \alpha)p(\alpha)d\alpha,
\end{equation*}
where $P(C_2|C_1, \alpha)$ can be bounded by writing $x=U_1\alpha+\textbf{n}$, and considering $\textbf{n}\in\cN(0,\sigma^2\textbf{I}).$ Similarly, $P(C_1|C_2)$ can be bounded. With the formalization above, the theorem below bounds the classification error:
\begin{theorem}[\textbf{Classification error bound}]\cite{subspaceangle}\label{theorem: principal angle classification error}
    As $\sigma^2\to 0$ the classification error, $P_e$ is upper bounded as $P_e\le \int \frac{1}{2}\exp{\left(-\frac{(\sum_{i=1}^d \sin^2\theta_i\alpha_i^2)^2}{8\sigma^2\sum_{i=1}^d \sin^2\theta_i(\alpha_i^2+\sigma^2)}\right)}\frac{p(\alpha)+q(\alpha)}{2}d\alpha,$ where $\{\cos (\theta_i)\}_{i=1}^d$ are the singular values of $U_1^\top U_2$ in ascending order, and $\{\theta_i\}_{i=1}^d$ are the principal angles  between $U_1$ and $U_2$. 
\end{theorem}
Theorem \ref{theorem: principal angle classification error} indicates that when the principal angles (or signal energy) are bigger, the classification error upper bound is smaller. When $\sigma^2\to 0,$ we have $\exp{\left(-\frac{(\sum_{i=1}^d \sin^2\theta_i\alpha_i^2)^2}{8\sigma^2\sum_{i=1}^d \sin^2\theta_i(\alpha_i^2+\sigma^2)}\right)}\approx \exp{\left(-\frac{\sum_{i=1}^d \sin^2\theta_i\alpha_i^2}{8\sigma^2}\right)}$, by ignoring the higher-order term of $\sigma^2$. This indicates that classification performance is a
function of principal angles and \emph{larger principal angles lead
to a smaller classification error.}

Consequently, we can formulate the positive and negative interactions between the tokens as a two-class classification problem. We introduce two queries over the same keys: (i) a \emph{\textcolor{red}{positive}} query capturing features belonging to the correct class, and (ii) a \emph{\textcolor{blue}{negative}} query capturing closely associated but irrelevant features. Their outputs are projected into distinct subspaces with large principal angles between them. Combining these projections suppresses non-essential patches (denoising) while maintaining small classification error through subspace separation.

Formally, the solution to the constrained minimization problem with a balancing parameter $\beta>0$:
\begin{equation}\label{eq:constrained_problem}
   \min_{\theta\in\Theta} \cL(\theta)+ \beta\sum_{i=1}^L\sum_{j=1,\\j\neq l}^H \langle V_{ij} (\theta),V_{il} (\theta)\rangle_F,   
\end{equation}
promotes orthogonality between value subspaces and yields a solution set contained in that of the unconstrained problem:
\begin{equation}\label{eq:unconstrained problem}
   \min_{\theta\in\Theta} \cL(\theta),
\end{equation}
where $\cL(\theta)$ is the training loss. Nevertheless, to obtain solutions to \eqref{eq:constrained_problem} one needs to guarantee $\sum_{i=1}^L\sum_{j=1}^H \langle V_{ij}, V_{ik}\rangle_F$ is small at every iterate. Enforcing orthogonality at every iterate is difficult. If we can argue that $\sum_{i=1}^L\sum_{j=1, j\neq l}^H \langle V_{ij}^{(k)},V_{il}^{(k)}\rangle_F$ is \emph{quasi-orthogonal} for all $k$ then it suffices to only solve \eqref{eq:unconstrained problem}.

In practice, ${(W_{V_{ij}})_{pq}\sim \cN(0,\sigma^2)}$, chosen from a zero mean Gaussian distribution with standard deviation ${\sigma=0.02}$. Since ${Support}(\cN(0,\sigma^2))=\R$, for all $\sigma\in\R^+,$ we consider a compact support. Additionally, we follow \cite{zhang2022neural}, which views convergence in neural network training from an invariant measure perspective. Especially, Theorem 4.3 in \cite{zhang2022neural} shows that as the training loss $\cL(\theta)$ stabilizes, the iterates (or the weights) lie in a compact region almost surely. We modify the result of Theorem 4.3 below. 

\begin{proposition}\label{proposition:compact}
   Let $W_{V_{ij}(0)}$ be initialized within a compact set $\cC_V:=\{W_{V_{ij}(0)}: \|W_{V_{ij}(0)}\|\le \gamma\}.$ Then for every $k,$ the iterate $W_{V_{ij}(k)}\in\cC_V$ with high probability. 
\end{proposition}
Proposition \ref{proposition:compact} implies that all iterates lie in a compact region almost surely. Based on this, we proposed the following Theorem, which states that, under certain standard assumptions, the solutions to \eqref{eq:unconstrained problem} are orthogonal with high probability. 

\begin{theorem}[\textbf{Quasi-orthogonality of the subspaces}]\label{theorem:subspace_near_orthogonal}
   Let the entries of $W_{V^{+}}$ and $W_{V^{-}}$ be initialized as i.i.d. $\cN(0,\sigma^2)$ over a finite support $[-\gamma, \gamma].$  Let the input $X_p\in\R^{N\times d}$ be such that $\|X_p\|_{\infty}\le M$. Assume for every $k,$ the entries of the iterates $\{W_{V^{+}{(k)}}, W_{V^{-}{(k)}}\}$ are i.i.d. sub Gaussian with parameter $\sigma>0$. \myNum{i} Then $|(W_{V^{+}{(k)}})_{pq}|\le \gamma$ and $|(W_{V^{-}{(k)}})_{pq}|\le \gamma$, with high probability. \myNum{ii} If $\gamma=O(M^{-1}{\epsilon}^{-\frac{1}{2}}(N\log(\delta^{-1}))^{-\frac{1}{4}})$, then $|\langle V^+,V^-\rangle_F| \le d\epsilon$ with probability $1-\delta.$
\end{theorem}
\smartparagraph{DnA: Modeling Discriminative Patches.}\label{sec:contrastive_2q2v} While differential attention may remove \emph{common-mode noise, it could also cancel useful signals if $Q_1K_1^\top$ and $Q_2K_2^\top$ both assign high attention weights to the same tokens}. Since both branches share a single $V$, subtracting them directly removes overlapping contributions. To avoid this, we introduce two distinct value projections, $V^{+} = X_pW_{V^+} \in \mathbb{R}^{N \times D}$ and $V^{-} = X_pW_{V^-} \in \mathbb{R}^{N \times D}$, where $W_{V^+}, W_{V^-} \in \mathbb{R}^{D \times D}$ are learnable weight matrices. Consequently, we partition $V^{+}$ and $V^{-}$ into $H$ heads such that, for any head $h$, we head-specific value subspaces, $V^{+}_h, V^{-}_h \in \mathbb{R}^{N \times d}$. With this formalization, our denoising attention takes the form:
\begin{equation*}
\label{equ:contr_2q2v}
\begin{aligned}
\cA^{Q^{\pm}V^{\pm}}_h = \boldsymbol{\sigma}\left(\frac{Q_h^{+} K_h^\top}{\sqrt{d}}\right)V_h^{+}+\alpha_h\boldsymbol{\hat{\sigma}}\left(\frac{Q_h^{-} K_h^\top}{\sqrt{d}}\right)V_h^{-},
\end{aligned}
\end{equation*}
where $\alpha_h\in\R$ is a learnable balancing parameter; see \Cref{fig:arch}(ii).~It contains all previously discussed designing choices: \myNum{i} Two {\emph quasi-orthogonal value spaces}; \myNum{ii} that are scaled by using softmax for \textcolor{red}{positive} query and softmin for \textcolor{blue}{negative} query, to capture complementary signal components, analogous to denoising or band-pass filtering \cite{christiano2003band, laplante2018comprehensive}; see \Cref{alg:denoising attention}. For ablations on different design choices of \attnname and their performances on ImageNet-1K~\cite{deng2009imnet}, see \S\ref{app:supp_analysis}.

\subsection{Adapting \attnname for Video Understanding}
\label{sec:denoising_foundation_model}

In this section, we explore how \attnname can be integrated into video understanding frameworks, including conventional transformer-based architectures (e.g., TimeSformer~\cite{timesformer}) and video foundation models~\cite{vllama3}. This adaptation enables us to evaluate \attnname on fine-grained video analysis tasks, where such an attention mechanism can be particularly beneficial by suppressing detrimental token interactions and improving representation quality.

\smartparagraph{Video Transformer}~\cite{timesformer} uses a dissociated space-time attention within each transformer block to capture discriminative action patterns. We integrate \attnname into this framework by replacing both the spatial and temporal self-attention mechanisms with our \attnname variant. This allows the mechanism to operate independently along both dimensions, filtering irrelevant spatial cues within each frame and temporal noise across frames. 

\smartparagraph{Video LLMs} involve a vision-language adapter, often consisting of simple MLPs or linear layers. Recently, visual probing~\cite{viscop} emerged as an additional visual reinforcement for the LLM.
The role of visual probes in VisCoP~\cite{viscop} is to provide a compact set of learnable tokens $\mathbf{P}$ that extract domain-specific spatio-temporal cues $\mathbf{X}$ from the hidden states of the vision encoder. These probes augment the frozen vision encoder embeddings before they are passed to the LLM. At each layer $l$, the visual probes are updated through a cross-attention operation between learnable queries and the visual embeddings.
We adapt \attnname to video LLMs by replacing the cross-attention modules in VisCoP with our proposed attention mechanism. To the best of our knowledge, this is the first integration of explicit negative token interaction modeling within a foundation model through an adapter-based design, allowing the model to retain the benefits of large-scale pretraining while leveraging the denoising capability of our attention mechanism.
Specifically, the denoising interaction module at layer $\ell$ introduces projection matrices $W_{Q^+}, W_{Q^-}, W_K, W_{V^+}, W_{V^-}$, such that the probe update becomes:
{\footnotesize
\begin{equation*}
\mathbf{P}^{\ell+1} = \boldsymbol{\sigma}\left(\frac{\mathbf{P}^{\ell} W_{Q^+}^{\ell}(\mathbf{X}^{\ell} W_K^{\ell})^\top}{\sqrt{d_v}}\right)(\mathbf{X}^{\ell} W_{V^+}^{\ell})
+ \alpha\,\boldsymbol{\hat{\sigma}}\left(\frac{\mathbf{P}^{\ell} W_{Q^-}^{\ell}(\mathbf{X}^{\ell} W_K^{\ell})^\top}{\sqrt{d_v}}\right)(\mathbf{X}^{\ell} W_{V^-}^{\ell}).
\end{equation*}
}
\noindent This \textit{cross-\attnname} formulation explicitly models interactions between positive and negative tokens, enabling the probes to capture fine-grained visual cues within the LLM embedding space.

\section{Benchmarking and Evaluation}
\label{sec:expt}

In this section, we benchmark \attnname in traditional transformers and a foundation model for diverse visual understanding tasks, and analyze the effectiveness of \attnname quantitatively and qualitatively.

\subsection{Denoising Attention for Image Classification}
\label{sec:expt_image}

\begin{figure}[t]
\TopFloatBoxes
\begin{floatrow}[2]
\centering

\floatbox[\footnotesize]{table}[0.64\textwidth][][t]{%
\begingroup
\centering
\tiny
\setlength{\tabcolsep}{0.5pt}
\resizebox{0.64\textwidth}{!}{
\begin{tabular}{lccccccc}
\toprule
\multirow{2}{*}{\textbf{Model}} & \multicolumn{2}{c}{\textbf{Validation Set}} & \multicolumn{2}{c}{\textbf{Test Set}} & \multirow{2}{*}{\textbf{Params}} & \multirow{2}{*}{\textbf{GFLOPs}} & \multirow{2}{*}{\textbf{Throughput}}\\
\cmidrule(lr){2-3}\cmidrule(lr){4-5}
& \textbf{Acc.@1} & \textbf{Acc.@5} & \textbf{Acc.@1} & \textbf{Acc.@5} & & &\\
\midrule
ViT-B & 81.1 & 95.6 & 81.1 & 95.6 & 86.6M & 17.6 & 909.1 img/s\\
\midrule
+Diff. Attn. & 81.4(\greenup0.3) & 95.7(\greenup0.1) & 81.5(\greenup0.4) & 95.6 & 86.6M & 17.6 & 909.1 img/s\\
+Cog Attn. & 81.4(\greenup0.3) & 95.7(\greenup0.1) & 81.5(\greenup0.4) & 95.7(\greenup0.1) &  86.6M & 17.6 & 909.1 img/s\\
\midrule
\rowcolor{Gray}
\textbf{+DnA} & \textbf{81.9}(\greenup0.8) & \textbf{95.9}(\greenup0.3) & \textbf{81.9}(\greenup0.8) & \textbf{95.8}(\greenup0.2) & 100.7M & 21.1 & 909.1 img/s\\
\bottomrule
\end{tabular}
}
\endgroup
}{
\caption{\small{Accuracy on the ImageNet-1K~\cite{deng2009imnet} validation and test set; the \greenup~arrows indicate the absolute gain compared to the baseline ViT-B.~We present the variance across multiple runs in \Cref{tab:imagenet1k_var}.~See~\S\ref{app:supp_implementation} for implementation details and hyperparameter setup, and computational requirements in \Cref{tab:supp_computation}. Also, see \S\ref{app:supp_training_details} for the training loss and test accuracy curves.}}
\label{tab:imagenet1k}
}

\floatbox[\footnotesize]{table}[0.34\textwidth][][t]{%
\begingroup
\centering
\tiny
\setlength{\tabcolsep}{0.1pt}
\resizebox{0.34\textwidth}{!}{
\begin{tabular}{lccc}
\toprule
\multirow{2}{*}{\textbf{Model}} & \multicolumn{3}{c}{\textbf{ImageNet-A~\cite{imagenetao}}} \\
\cmidrule(lr){2-4}
& {Acc.$\uparrow$} 
& {RMSE$\downarrow$} 
& {AURRA$\uparrow$} \\
\midrule
ViT-B & 24.1 & 26.6 & 33.5 \\
\midrule
+Diff. Attn. & \textbf{28.1} & 25.8 & 39.7 \\
+Cog Attn. & 27.5 & 25.3 & 38.8 \\
\midrule
\rowcolor{Gray}
\textbf{+DnA} & 27.7 & \textbf{24.2} & \textbf{40.7} \\
\bottomrule
\end{tabular}
}
\endgroup
}
{
\caption{\small{Robustness evaluation on ImageNet-A~\cite{imagenetao}. Improvements in ranking and calibration metrics reflect decision quality on adversarial samples.}}
\label{tab:imageneta}
}
\end{floatrow}
\end{figure}

\smartparagraph{Setup.}
For Image Classification, we train our models and baselines from scratch on ImageNet-1K~\cite{deng2009imnet} with images resized to $224\times 224$.
We implement \attnname on the ViT-B backbone, using hyperparameters from DeiT \cite{deit}. Our models contain 12 encoder layers, with 12 heads in each layer, a patch size of $p=16$, and a head dimension of 64. We train them on an NVIDIA H100 GPU for 300 epochs; see implementation details and training hyperparameters in Table \ref{tab:supp_imnet_hp} in~\S\ref{app:supp_implementation}.

\smartparagraph{Baselines.}
Our primary baseline is ViT-B, trained with a DeiT recipe \cite{deit}. For the image classification on ImageNet-1K, we employ cog attention~\cite{cogformer} and differential attention~\cite{yedifferential} in ViT-B by replacing traditional MHSA. The rest of the ViT architecture remains unchanged. Cog attention \cite{cogformer} employs negative attention weights, thereby providing a direct point of comparison with our approach. Differential attention \cite{yedifferential} employs a signal denoising mechanism for language modeling tasks by splitting keys and queries. See implementation details in \S\ref{app:supp_implementation}.

\smartparagraph{Evaluation.}
\Cref{tab:imagenet1k} shows the evaluation results of our models on the ImageNet-1K \cite{deng2009imnet} validation and test set compared to the baselines. The proposed \attnname surpasses the baseline ViT-B by 0.8\% and outperforms models using cog attention and differential attention by 0.5\%, \emph{while maintaining the same throughput during inference}. 

\smartparagraph{Robustness Evaluation on ImageNet-A.}
ImageNet-A~\cite{imagenetao} consists of images that are full of adversarial examples and that confuse the image classifiers. We use this dataset to evaluate the robustness of our model to such adversarial examples. For evaluation, we use the calibration error metric RMSCE, the reliability metric AURRA, and standard accuracy. We show these results in Table \ref{tab:imageneta}. In terms of reliability and calibration, our proposed \attnname substantially outperforms the baselines and competitors.

\smartparagraph{Additional Benchmarking.}
We use the weights of our ImageNet-1K trained models for downstream visual understanding tasks: \noindent\myNum{i} \emph{image classification} on CIFAR-10, CIFAR-100 \cite{krizhevsky2009cifar10_100}, and Stanford Cars \cite{stanfordcars}, \myNum{ii} \emph{object detection and instance segmentation} on MS COCO \cite{mscoco}. We further conduct \myNum{iii} \emph{robustness evaluations} on ImageNet-O, -R \cite{imagenetao,imagenetr}; see \S\ref{app:additional_benchmarking} and Tables~\ref{tab:transfer_learning}--\ref{tab:imagenet_var_results}. We incorporate \attnname for a few additional backbones and show their performances in Table \ref{tab:accuracy_comparison}.

\subsection{Denoising Attention for Video Understanding}
\label{sec:expt_vid_und}

We evaluate \attnname on fine-grained video understanding tasks to demonstrate the importance of specialized attention mechanisms for detailed video analysis. Specifically, we integrate \attnname in two settings. First, we incorporate \attnname into a conventional \textit{Video Transformer}~\cite{timesformer} for action classification on exocentric videos. Second, we integrate \attnname into a \textit{video LLM}~\cite{vllama3} to address egocentric video understanding tasks.

\begin{table*}[t]
\centering
\begin{subtable}[t]{0.55\textwidth}
\centering
\resizebox{\linewidth}{!}{%
\begin{tabular}{@{} l|cc|cc @{}}
\toprule
 & \multicolumn{2}{c|}{\textbf{Toyota Smarthome}\cite{das2019toyotasmarthome}} & \multicolumn{2}{c|}{\textbf{NTU60}\cite{shahroudy2016ntu}}  \\
 \textbf{Model} & \textbf{CS} & \textbf{CV2} & \textbf{CS} & \textbf{CV}  \\
\midrule
TimeSformer & 67.5 & 59.5 & 81.2 & 88.6  \\
+Diff. Attn. & 66.6 & 59.4 & 81.5 & 89.1 \\
\rowcolor{Gray}
\textbf{+\attnname} & \textbf{68.8} & \textbf{63.5} & \textbf{82.4} & \textbf{89.2}  \\
\bottomrule
\end{tabular}%
}
\end{subtable}%
\hfill
\begin{subtable}[t]{0.42\textwidth}
\centering
\resizebox{\linewidth}{!}{%
\begin{tabular}{@{} l|cccc|c @{}}
\toprule
 \textbf{Model} & \textbf{Act.} & \textbf{Task} & \textbf{HOI} & \textbf{Hand} & \textbf{Avg.} \\
\midrule
VideoLLaMA3 & 74.9 & 75.8 & 75.2 & \textbf{65.3} & 72.8 \\
VisCoP & 81.8 & 86.1 & \textbf{79.3} & 65.1 & 78.1 \\
+Diff. Attn. & 78.8 & 83.0 & 77.5 & 64.9 & 76.1 \\
\rowcolor{Gray}
\textbf{+\attnname} & \textbf{83.1} & \textbf{87.1} & 79.2 & 65.1 & \textbf{78.6} \\
\bottomrule
\end{tabular}%
}
\end{subtable}%
\caption{\small{\textbf{(a)}\textbf{Left:} Baseline comparison on Toyota Smarthome and NTU60. We report mean-class accuracy (mCA) for Toyota Smarthome and Top-1 accuracy for NTU60. Cross-subject (CS) and cross-view (CV) denote the data split protocols. \textbf{(b)} \textbf{Right:} Quantitative results on Ego-in-Exo PerceptionMCQ~\cite{reilly2025egoexo}. We evaluate across four egocentric video understanding tasks: action understanding (Act.), task-relevant region understanding (Task), human-object interactions (HOI), and hand identification (Hand).\vspace{-7mm}}}
\label{tab:video_results}
\end{table*}

\smartparagraph{Video Transformer}~\cite{timesformer}. We initialize the weights from the trained models in \S\ref{sec:expt_image}, and pretrain the models on Kinetics400~\cite{Carreira_2017_CVPR}, then finetune and evaluate them on Toyota Smarthome~\cite{das2019toyotasmarthome} and NTU RGB+D 60~\cite{shahroudy2016ntu}. Toyota Smarthome consists of 16K videos spanning 31 action categories. NTU RGB+D 60 contains 57K video samples across 60 action classes. 
Models are trained for 15 epochs using divided space-time attention, with an input of 8 frames at a resolution of $224 \times 224$. For Toyota Smarthome, we evaluate using the cross-subject (CS) and cross-view (CV2) protocols, reporting mean class accuracy (mCA). For NTU60, we report Top-1 accuracy. 

\smartparagraph{Results.} Table~\ref{tab:video_results}(a) shows that replacing both the space and temporal attention in TimeSformer with our denoising mechanism consistently improves performance across both datasets. On Toyota Smarthome~\cite{das2019toyotasmarthome}, we observe gains of \greenup $1.3\%$ on CS and \greenup $4.0\%$ on CV2. On NTU60~\cite{shahroudy2016ntu}, denoising attention achieves improvements of \greenup $1.2\%$ on CS and \greenup $0.6\%$ on CV.

Overall, these consistent gains across two datasets show the effectiveness of incorporating denoising attention in video transformers.

\smartparagraph{Video LLM.}~\label{sec:video_llm} Following the training strategy of VisCoP~\cite{viscop} and by using a subset of ego videos from EgoExo4D~\cite{cvpr2025egoexo4d} ($\sim$46K QA pairs), we train VideoLLaMA3 with \attnname as the attention mechanism and evaluate on the Ego-in-Exo PerceptionMCQ~\cite{reilly2025egoexo} benchmark. This benchmark comprises four evaluation categories: \textit{Action Understanding}, which measures recognition of performed activities; \textit{Task Regions}, which assesses identification of spatially salient areas in a scene; \textit{Human–Object Interaction} (HOI), which evaluates reasoning about interactions between individuals and objects; and \textit{Hand Identification}, which tests first-person perspective understanding through hand tracking and recognition. 

We initialize the positive projection weights from the pretrained vision encoder's cross-attention weights, and initialize $W_{Q^-}^{\ell}$ and $W_{V^-}^{\ell}$ as scaled-down copies of their positive counterparts (by $10^{-4}$). The scaling factor avoids branch symmetry, which would arise from two equivalent initialization points~\cite{chen2022bert2bert,chen2015net2net}; see~\S\ref{app:branch_init} and Table~\ref{tab:video_llm_scale}. During finetuning, we train the interaction modules, visual probe projector, vision embedding projector, and apply LoRA to the LLM. We train for 3 epochs with an initial learning rate of $10^{-5}$, using a cosine schedule for the projectors and the LLM. Evaluation uses a temperature of 0.

\smartparagraph{Results.}
Table~\ref{tab:video_results}(b) displays our results on Ego-in-Exo PerceptionMCQ~\cite{reilly2025egoexo}. Our denoising variant improves over the baseline VisCoP on average, achieving gains of \greenup $0.5\%$, showing that the denoising mechanism helps the visual probes extract more semantically meaningful domain-specific cues.

\begin{figure*}[t]
    \centering
    \begin{subfigure}{\linewidth}
        \centering
        \includegraphics[width=\linewidth]{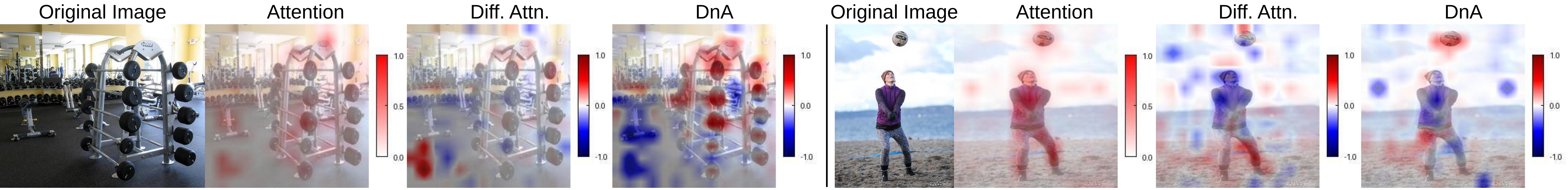}
    \end{subfigure}
    \begin{subfigure}{\linewidth}
        \centering
        \includegraphics[width=\linewidth]{figs/vid_understanding_mcq_v2.pdf}
    \end{subfigure}
    \caption{\small{\textbf{(a)} \textbf{Top:} Attention visualization for ViT-B~\cite{vit} with softmax, differential attention~\cite{yedifferential}, and our \attnname. The objects are dumbbells and a volleyball. \textbf{(b)} \textbf{Bottom:} Example of \attnname outperforming baseline methods on egocentric video QA. Four frames are uniformly sampled; the object of interest is highlighted by a cyan box. The correct answer is in \textcolor{correctgreen}{\textbf{green}}, the incorrect in \textcolor{red}{\textbf{red}}.}}
    \label{fig:visualization}
\end{figure*}

\subsection{Qualitative Analysis: Attention Visualization}
\label{sec:visualization}

We use the gradient-based attention interpretability method from~\cite{chefer2021transformer} for visualization on ViT-B.  For \attnname and differential attention~\cite{yedifferential}, we extend this method to propagate and visualize signed relevance independently.~\Cref{fig:visualization}(a) demonstrates that softmax attention is densely spread around the image, but with limited focus on the classes of interest. Differential attention, on the other hand, shows improvement but retains substantial background noise. Finally, the positive activations (red) in \attnname are focused on the classes of interest (dumbbell and volleyball). Crucially, the $\mathcal{A}^{Q^-V^-}_h$ branch of \attnname suppresses irrelevant background areas, creating a contrast between class features and noisy distractors.~\Cref{fig:visualization}(b) displays how \attnname thrives despite the presence of many visually prominent objects, by suppressing query-irrelevant tokens before passing them to the LLM, whereas VisCoP and the Base VLM are mislead by irrelevant visual cues. We provide additional visualizations in \S\ref{app:supp_analysis}; see Figures~\ref{fig:additional_visualizations} and \ref{fig:additional_ego_vis}.

\subsection{Quantitative Analyses: Understanding the Attention Subspaces}\label{sec:subspace_analysis}

This section empirically justifies our postulates in \S\ref{sec:denoising_attention} that lead to the design choices of \attnname. First, by calculating the intruder dimensions between two attention branches of the differential transformer and \attnname, we show the necessity of considering two value subspaces. Next, by empirically measuring the similarities between the contrasting value subspaces, $V^+$ and $V^-$ of \attnname, we verify the claim of our main theoretical result in Theorem \ref{theorem:subspace_near_orthogonal}.

\myNum{a}\smartparagraph{Intruder Dimension Analysis.} To observe how two branches in \attnname relate to one another, we perform a subspace analysis that measures the divergence between $\mathcal{A}^{Q^+V^+}_h$ and $\mathcal{A}^{Q^-V^-}_h$ in \attnname and $\cA_1V$ and $\cA_2V$ of the differential transformer. We compute the SVD of each branch and count how many of the Top-$k$ left singular vectors are \emph{near-orthogonal.} The near orthogonality of these vectors indicates the separability of the subspaces and hence the efficacy of denoising attention in reducing the classification error; see Theorem \ref{theorem: principal angle classification error}. Out of the Top-$k$ singular vectors, we called the vectors with larger angles between them \emph{intruders}, following \cite{shuttleworth2025lora, beyondlora}; the higher the intruders, the better. 

We calculated the intruders independently per sample, head, and layer across 5 images per class from the ImageNet-1K validation set ($5$K images total), resulting in $720$K values for ViT-B.  Out of the Top-$k$ left singular vectors, a singular vector is considered an intruder if its cosine similarity with the rest of the singular vectors falls below a user-defined threshold, $\epsilon$. 

Figure~\ref{fig:ortho_hist} shows \attnname consistently displays higher intruder dimensions, with its distribution peaking around $8-9$, compared to differential attention, which peaks around $6-7$ (we set $\epsilon=\cos\frac{\pi}{3}, k=10$). This shows that in differential attention, two attention branches operate within the same subspace, performing a local refinement that may cancel common-mode noise, where \attnname's branches interact more orthogonally. Since the subspaces are near-orthogonal, \attnname can suppress noise without canceling any useful signal and has lower classification error, as seen in the visual understanding tasks. On the other hand, branches $\cA_1$ and $\cA_2$ in differential attention may both assign high attention weights to the same tokens, losing a useful signal in the subtraction.

\begin{figure}[t]
\begin{floatrow}[2]
\centering

    \ffigbox[0.495\textwidth]{%
        \includegraphics[width=\linewidth]{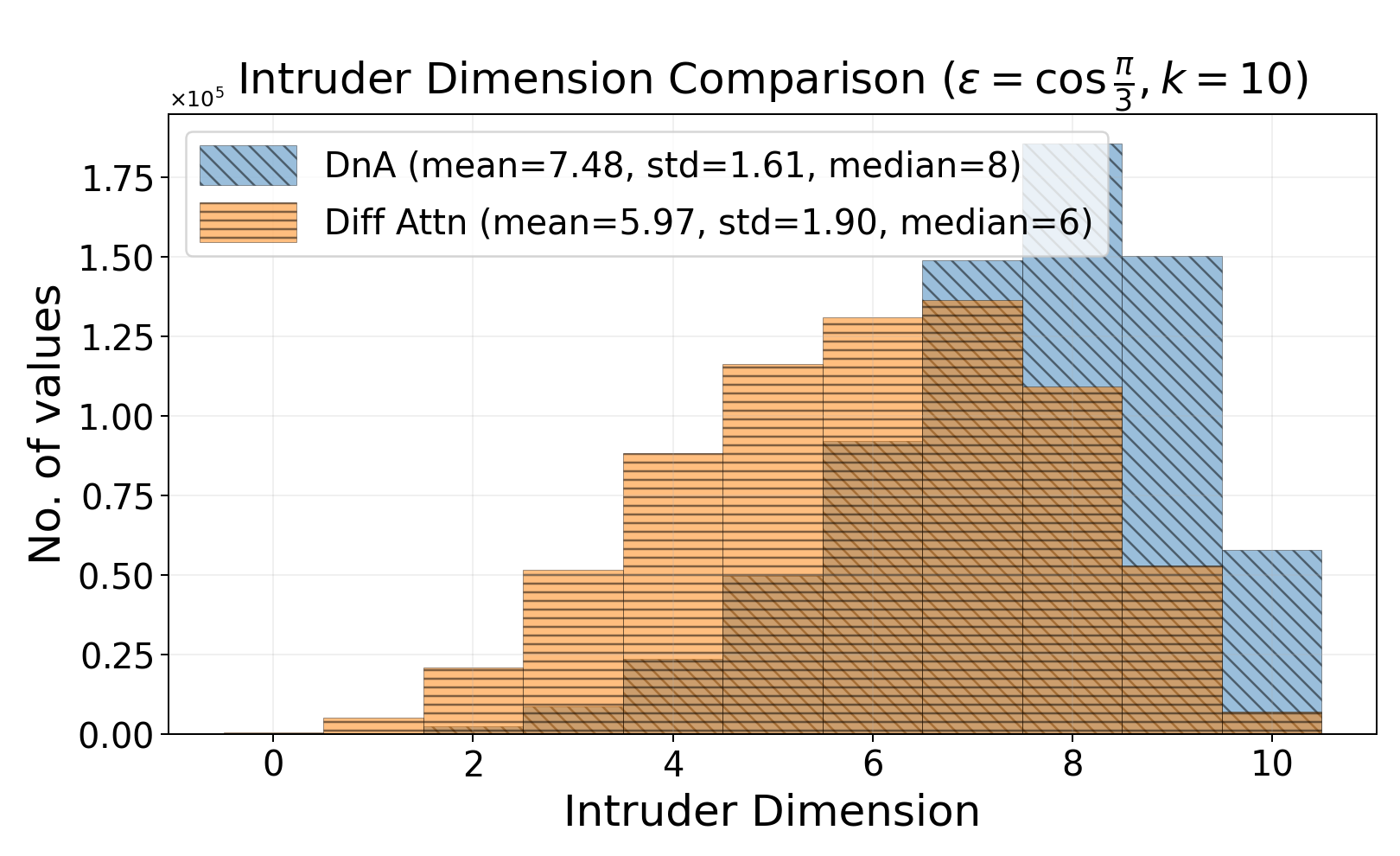}
        }{%
        \caption{\small{Comparison of intruder dimension counts between \attnname and differential attention $\epsilon=\cos\frac{\pi}{3}$, $k{=}10$.}
        }
        \label{fig:ortho_hist}
        }
    
    \ffigbox[0.495\textwidth]{%
        \includegraphics[width=\linewidth]{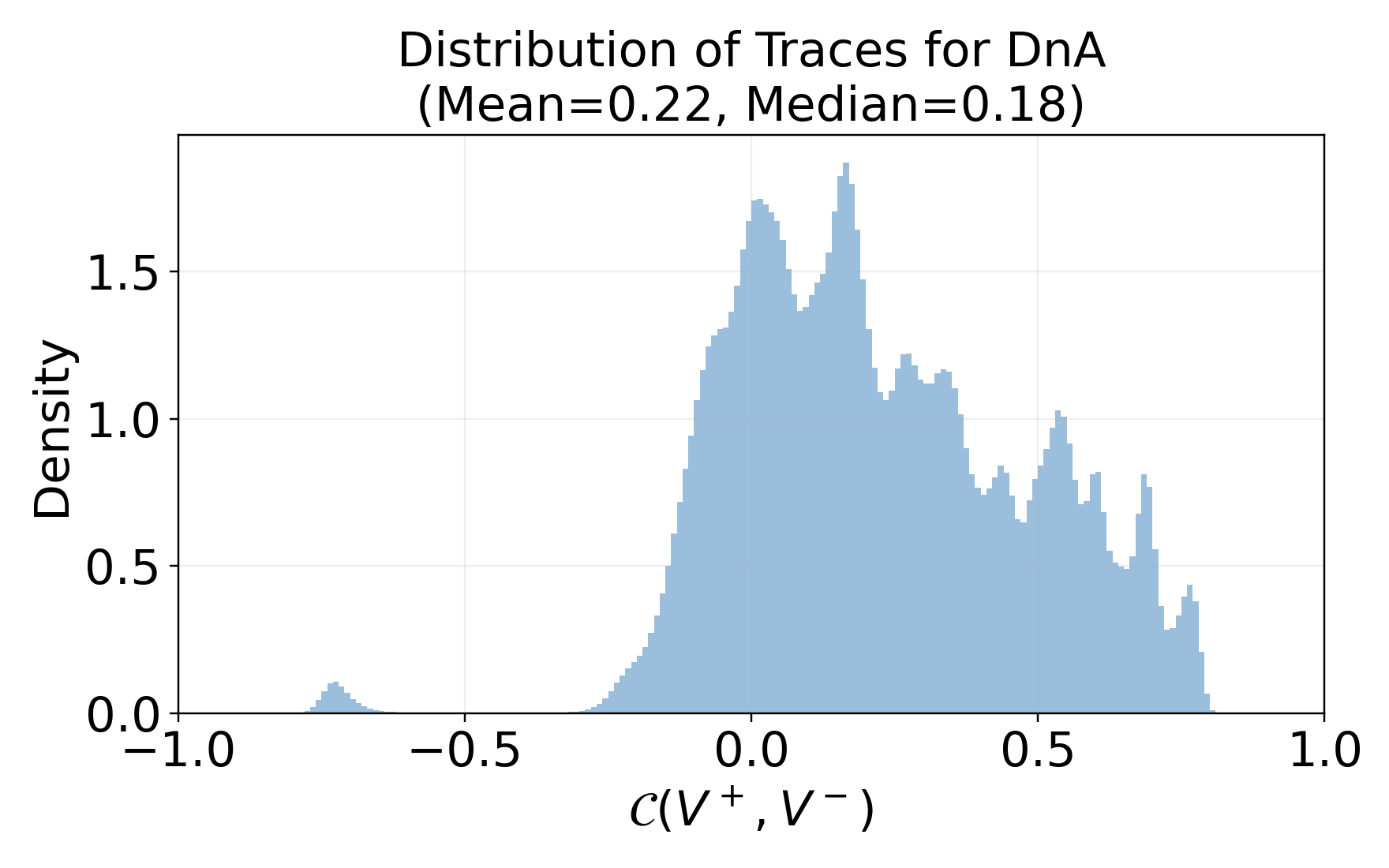}
        }{%
        \caption{\small{Distribution of $\cC({V^{+}, V^{-}})$ for each sample, head, and layer in the ImageNet-1K validation set ($7.2\text{M}$ trace values).}
        }
        \label{fig:trace_hist}
        }
\end{floatrow}
\end{figure}

\myNum{b}\smartparagraph{Similarity between $V^+$ and $V^-$.} We consider the normalized Frobenius inner product, $\cC({V^{+}, V^{-}})$ to calculate the similarities between two value subspaces, given by $\cC({V^{+}, V^{-}}):=\langle \frac{V^{+}}{\| V^{+} \|_F}, \frac{V^-}{\| V^{+} \|_F} \rangle_F=\mathrm{Trace} \left(\frac{V^{+^{\top}} V^{-}}{\| V^{+} \|_F \| V^{-} \|_F} \right).$ We normalize the matrices so that the range is $[-1,1]$, where $1$ indicates that the two matrices are aligned, and $0$ indicates orthogonality. We plot $\cC({V^{+}, V^{-}})$ for each sample, head, and layer using the ImageNet-1K validation set, resulting in 7.2M trace values; see Figure \ref{fig:trace_hist}. The mean of the distribution is $0.22$, while the median is $0.18$. Figure~\ref{fig:trace_hist} shows that the subspaces $V^{+}$ and $V^{-}$ are quasi-orthogonal for most heads, with most of the mass near $0$. However, the distribution shows a heavy right tail, suggesting that some attention heads in $V^{+}$ and $V^-$ may not contribute equally to the denoising mechanism. \emph{This empirical observation justifies our theoretical result in Theorem} \ref{theorem:subspace_near_orthogonal}. It also supports the claim that, instead of enforcing orthogonality at every iteration (as in the case of \eqref{eq:constrained_problem}), it is sufficient to solve the unconstrained problem \eqref{eq:unconstrained problem}, as the resulting subspaces are quasi-orthogonal during training. 

Additionally, we compute the average cosine similarity between $\mathcal{A}^{Q^+V^+}_h$ and $\mathcal{A}^{Q^-V^-}_h$ for \attnname and $\cA_1 V$ and $\cA_2 V$ for differential attention, over the entire ImageNet-1K validation set. Differential attention achieves a high average cosine similarity of $0.96$, while \attnname measures at only $0.32$, further solidifying our claim.

\subsection{Additional Quantitative Analyses \& Ablations}\label{sec:analysis} We further examine \attnname from multiple perspectives through extensive ablations.
 
\smartparagraph{Representational Robustness.} We measure representational robustness by computing $\Delta_{ij}:=1-\cos(\angle p_i,p_j)$ between two tokens, $p_i,p_j\in\cA,$ at each layer. Before calculating $\Delta_{ij}$, token embeddings (excluding the CLS token) are mean-centered across feature space, then averaged per image and across the dataset. The higher this measurement, the better, as it indicates greater inter-token separation. To further substantiate our claim that our attention mechanism denoises, we infuse the input with Gaussian noise.~\Cref{fig:distance_feature_space_noise} shows how \attnname stabilizes noisy inputs through layer-wise features. Overall, \attnname improves robustness by suppressing perturbations and transforming them into features that downstream layers can process effectively. In contrast, softmax attention propagates perturbations, amplifying distortions in feature space and degrading accuracy. This denoising behavior lets \attnname\ recover from early noise more reliably.

\begin{figure}[t]
    \centering
    \includegraphics[width=\linewidth]{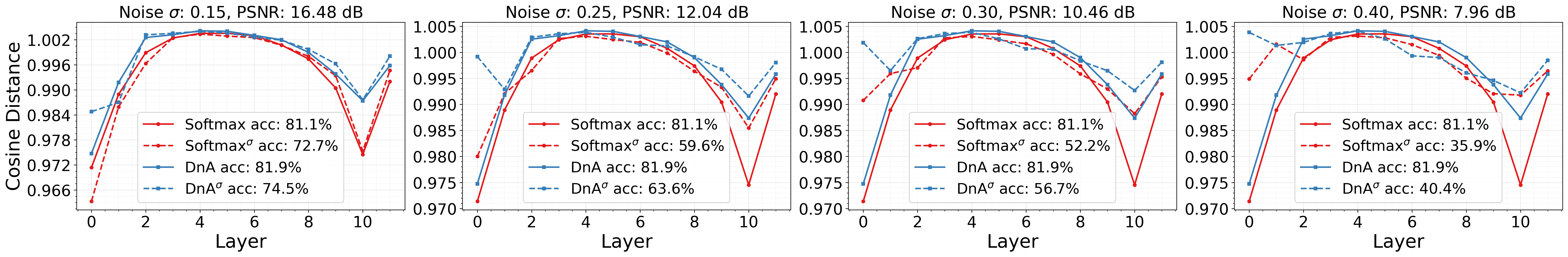}
    \caption{\small{Average pairwise cosine distances under varying noise levels; see noise variance ($\sigma$) and PSNR in titles. Solid lines represent models evaluated under clean data (\attnname and softmax), while dashed lines reflect artifacts of noisy data (\attnsigma and softmax$^\sigma$) for different noise, $\sigma$. \attnsigma exhibits smaller accuracy drops than softmax$^\sigma$ across all noise levels, indicating more diverse feature representations and robustness to noise.}}
    \label{fig:distance_feature_space_noise}
\end{figure}

\smartparagraph{Additional Analysis.} In \S\ref{app:supp_analysis}, we trained two parameter-matched ViT-B baselines, and their performance confirms that \attnname's gains stem from the architectural choices rather than extra parameters. Additionally, we discuss different design components of \attnname and their behavior in \S\ref{app:supp_analysis}. We further analyze attention head diversity, the evolution of the learnable parameter ($\alpha_h$), and the entropy of attention distributions across encoder layers in Figure~\ref{fig:entropy_depth}, \S\ref{app:supp_analysis}.

\section{Limitation \& Future Work}

\attnname introduces additional parameters, ranging from approximately 3\% in video LLMs to 23\% in video transformers. While the performance improvements stem from the discriminative capability of the proposed attention mechanism rather than simply increased parameter capacity (see Table~\ref{tab:ablation_qv}), we acknowledge the added parameters and associated training cost. In practice, all the additional parameters reside in parallelizable attention heads, which mitigate the computational overhead. Consequently, we observe throughput comparable to a baseline ViT when compared with ViT augmented with \attnname (see Table~\ref{tab:supp_computation}). Future work will focus on developing more parameter-efficient implementations while preserving the performance gains of the proposed mechanism.
\section{Conclusion}
\label{sec:conc}

We propose \attnname, denoising attention that uses softmin and softmax and explicitly models negative token interactions. Our design principle for \attnname is motivated by theoretical observations that indicate that subspace separation for the positive and negative token interactions can alleviate both attention misallocation and noise. By replacing traditional attention, the proposed \attnname outperforms standard ViT-B on ImageNet-1K classification, improves on several video understanding tasks with both video transformers and video LLMs, and generalizes to multiple downstream tasks. 
% To justify the design choices, we further verify that the two branches in \attnname, related to the two different interacting subspaces, learn distinct token interactions that result in the denoising effect.
To justify the design choices, we verify that the two branches in \attnname, operating over distinct interacting subspaces, learn different token interactions that jointly enable the denoising effect.

% In the future, we plan to extend \attnname for domain-specific tasks, such as language modeling.

%to other domains such as NLP, as well as alternative optimization strategies which may be better suited for our attention paradigm. 
%We hope our work inspires closer looks at the attention mechanism and further improvements to it

\subsubsection*{Acknowledgments.}
Aritra Dutta is partially supported by the Florida Department of Health Grant, AWD00007072, and the National Science Foundation Grant, 2321986. Subhajit Maity is supported by the Florida Department of Health Grant AWD00007072. Srijan Das is in part supported by the National Science Foundation, IIS-2245652. We thank Dominick Reilly for insightful discussions and technical assistance.

\bibliographystyle{splncs04}
\bibliography{main}

\begin{thebibliography}{10}
\providecommand{\url}[1]{\texttt{#1}}
\providecommand{\urlprefix}{URL }
\providecommand{\doi}[1]{https://doi.org/#1}

\bibitem{arnab2021vivit}
Arnab, A., Dehghani, M., Heigold, G., Sun, C., Lucic, M., Schmid, C.: {ViViT: A Video Vision Transformer}. In: Proceedings of the IEEE/CVF International Conference on Computer Vision. pp. 6836--6846 (2021)

\bibitem{bahdanau2014neural}
Bahdanau, D., Cho, K., Bengio, Y.: {Neural Machine Translation by Jointly Learning to Align and Translate}. In: International Conference on Learning Representations (2015)

\bibitem{timesformer}
Bertasius, G., Wang, H., Torresani, L.: {Is Space-Time Attention All You Need for Video Understanding?} In: International Conference on Machine Learning. pp. 833--842 (2021)

\bibitem{gpt3}
Brown, T., Mann, B., Ryder, N., Subbiah, M., Kaplan, J.D., Dhariwal, P., Neelakantan, A., Shyam, P., Sastry, G., Askell, A., et~al.: {Language Models are Few-Shot Learners}. In: Advances in Neural Information Processing Systems. pp. 1877--1901 (2020)

\bibitem{beyondlora}
Cadenhead, E., McGee, C., Li, X., Bergou, E.H., Dutta, A.: {Beyond LoRA: Is Sparsity-Induced Adaptation Better?} arXiv preprint arXiv:2606.13767  (2026)

\bibitem{campos2025gaea}
Campos, R., Vayani, A., Kulkarni, P.P., Gupta, R., Zafar, A., Dutta, A., Shah, M.: {GAEA: A Geolocation Aware Conversational Model}. In: Proceedings of the IEEE/CVF Winter Conference on Applications of Computer Vision. pp. 5236--5246 (2026)

\bibitem{Carreira_2017_CVPR}
Carreira, J., Zisserman, A.: {Quo Vadis, Action Recognition? A New Model and the Kinetics Dataset}. In: Proceedings of the IEEE/CVF Conference on Computer Vision and Pattern Recognition. pp. 6299--6308 (2017)

\bibitem{sqlbook}
Chamberlin, D.D., Boyce, R.F.: {SEQUEL: A Structured English Query Language}. In: Proceedings of the ACM SIGFIDET (Now SIGMOD) Workshop on Data Description, Access and Control. pp. 249--264 (1974)

\bibitem{chefer2021transformer}
Chefer, H., Gur, S., Wolf, L.: {Transformer Interpretability Beyond Attention Visualization}. In: Proceedings of the IEEE/CVF Conference on Computer Vision and Pattern Recognition. pp. 782--791 (2021)

\bibitem{chen2022bert2bert}
Chen, C., Yin, Y., Shang, L., Jiang, X., Qin, Y., Wang, F., Wang, Z., Chen, X., Liu, Z., Liu, Q.: {bert2{BERT}: Towards Reusable Pretrained Language Models}. In: Annual Meeting of the Association for Computational Linguistics. pp. 2134--2148 (2022)

\bibitem{chen2015net2net}
Chen, T., Goodfellow, I., Shlens, J.: {Net2net: Accelerating Learning via Knowledge Transfer}. In: International Conference on Learning Representations (2016)

\bibitem{chen2023pali}
Chen, X., Wang, X., Changpinyo, S., Piergiovanni, A., Padlewski, P., Salz, D., Goodman, S., Grycner, A., Mustafa, B., Beyer, L., Kolesnikov, A., Puigcerver, J., Ding, N., Rong, K., Akbari, H., Mishra, G., Xue, L., Thapliyal, A.V., Bradbury, J., Kuo, W., Seyedhosseini, M., Jia, C., Ayan, B.K., Ruiz, C.R., Steiner, A.P., Angelova, A., Zhai, X., Houlsby, N., Soricut, R.: {Pa{LI}: A Jointly-Scaled Multilingual Language-Image Model}. In: International Conference on Learning Representations (2023)

\bibitem{christiano2003band}
Christiano, L.J., Fitzgerald, T.J.: {The Band Pass Filter}. International Economic Review  \textbf{44}(2),  435--465 (2003)

\bibitem{cover2006elements}
Cover, T.M., Thomas, J.A.: {Elements of Information Theory}. John Wiley \& Sons, 2 edn. (2006)

\bibitem{darcet2024registers}
Darcet, T., Oquab, M., Mairal, J., Bojanowski, P.: {Vision Transformers Need Registers}. In: International Conference on Learning Representations (2024)

\bibitem{das2019toyotasmarthome}
Das, S., Dai, R., Koperski, M., Minciullo, L., Garattoni, L., Bremond, F., Francesca, G.: {Toyota Smarthome: Real-World Activities of Daily Living}. In: Proceedings of the IEEE/CVF International Conference on Computer Vision. pp. 833--842 (2019)

\bibitem{deng2009imnet}
Deng, J., Dong, W., Socher, R., Li, L.J., Li, K., Fei-Fei, L.: {ImageNet: A Large-Scale Hierarchical Image Database}. In: Proceedings of the IEEE/CVF Conference on Computer Vision and Pattern Recognition (2009)

\bibitem{ding2022davit}
Ding, M., Xiao, B., Codella, N., Luo, P., Wang, J., Yuan, L.: {DaViT: Dual Attention Vision Transformers}. In: Proceedings of the European Conference on Computer Vision. pp. 74--92 (2022)

\bibitem{vit}
Dosovitskiy, A., Beyer, L., Kolesnikov, A., Weissenborn, D., Zhai, X., Unterthiner, T., Dehghani, M., Minderer, M., Heigold, G., Gelly, S., Uszkoreit, J., Houlsby, N.: {An Image is Worth 16x16 Words: Transformers for Image Recognition at Scale}. In: International Conference on Learning Representations (2021)

\bibitem{dascoli2021convit}
d’Ascoli, S., Touvron, H., Leavitt, M.L., Morcos, A.S., Biroli, G., Sagun, L.: {ConViT: Improving Vision Transformers with Soft Convolutional Inductive Biases}. In: International Conference on Machine Learning. pp. 2286--2296 (2021)

\bibitem{goodfellow2016deep}
Goodfellow, I., Bengio, Y., Courville, A.: {Deep Learning}. MIT Press (2016)

\bibitem{cvpr2025egoexo4d}
Grauman, K., Westbury, A., Torresani, L., Kitani, K., Malik, J., Afouras, T., Ashutosh, K., Baiyya, V., Bansal, S., Boote, B., Byrne, E., Chavis, Z., Chen, J., Cheng, F., Chu, F.J., Crane, S., Dasgupta, A., Dong, J., Escobar, M., Forigua, C., Gebreselasie, A., Haresh, S., Huang, J., Islam, M.M., Jain, S., Khirodkar, R., Kukreja, D., Liang, K.J., Liu, J.W., Majumder, S., Mao, Y., Martin, M., Mavroudi, E., Nagarajan, T., Ragusa, F., Ramakrishnan, S.K., Seminara, L., Somayazulu, A., Song, Y., Su, S., Xue, Z., Zhang, E., Zhang, J., Castillo, A., Chen, C., Fu, X., Furuta, R., Gonzalez, C., Gupta, P., Hu, J., Huang, Y., Huang, Y., Khoo, W., Kumar, A., Kuo, R., Lakhavani, S., Liu, M., Luo, M., Luo, Z., Meredith, B., Miller, A., Oguntola, O., Pan, X., Peng, P., Pramanick, S., Ramazanova, M., Ryan, F., Shan, W., Somasundaram, K., Song, C., Southerland, A., Tateno, M., Wang, H., Wang, Y., Yagi, T., Yan, M., Yang, X., Yu, Z., Zha, S.C., Zhao, C., Zhao, Z., Zhu, Z., Zhuo, J., Arbelaez, P., Bertasius, G., Crandall, D.,
  Damen, D., Engel, J., Farinella, G.M., Furnari, A., Ghanem, B., Hoffman, J., Jawahar, C.V., Newcombe, R., Park, H.S., Rehg, J.M., Sato, Y., Savva, M., Shi, J., Shou, M.Z., Wray, M.: {Ego-Exo4D: Understanding Skilled Human Activity from First- and Third-Person Perspectives}. In: Proceedings of the IEEE/CVF Conference on Computer Vision and Pattern Recognition. pp. 19383--19400 (2024)

\bibitem{guo2019star}
Guo, Q., Qiu, X., Liu, P., Shao, Y., Xue, X., Zhang, Z.: Star-{Transformer}. In: Conference of the North American Chapter of the Association for Computational Linguistics. pp. 1315--1325 (2019)

\bibitem{hanbridging}
Han, D., Pu, Y., Xia, Z., Han, Y., Pan, X., Li, X., Lu, J., Song, S., Huang, G.: {Bridging the {Divide}: Reconsidering {Softmax} and {Linear} {Attention}}. In: Advances in Neural Information Processing Systems. pp. 79221--79245 (2024)

\bibitem{he2017mask}
He, K., Gkioxari, G., Doll{\'a}r, P., Girshick, R.: {Mask R-CNN}. In: Proceedings of the IEEE/CVF International Conference on Computer Vision. pp. 2961--2969 (2017)

\bibitem{he2016resnet}
He, K., Zhang, X., Ren, S., Sun, J.: {Deep Residual Learning for Image Recognition}. In: Proceedings of the IEEE/CVF Conference on Computer Vision and Pattern Recognition. pp. 770--778 (2016)

\bibitem{imagenetr}
Hendrycks, D., Basart, S., Mu, N., Kadavath, S., Wang, F., Dorundo, E., Desai, R., Zhu, T., Parajuli, S., Guo, M., et~al.: {The Many Faces of Robustness: A Critical Analysis of Out-of-Distribution Generalization}. In: Proceedings of the IEEE/CVF International Conference on Computer Vision. pp. 8340--8349 (2021)

\bibitem{imagenetao}
Hendrycks, D., Zhao, K., Basart, S., Steinhardt, J., Song, D.: {Natural Adversarial Examples}. In: Proceedings of the IEEE/CVF Conference on Computer Vision and Pattern Recognition. pp. 15262--15271 (2021)

\bibitem{subspaceangle}
Huang, J., Qiu, Q., Calderbank, R.: {The Role of Principal Angles in Subspace Classification}. IEEE Transactions on Signal Processing  \textbf{64}(8),  1933--1945 (2016)

\bibitem{hyeon2023scratching}
Hyeon-Woo, N., Yu-Ji, K., Heo, B., Han, D., Oh, S.J., Oh, T.H.: {Scratching Visual Transformer's Back with Uniform Attentionz}. In: Proceedings of the IEEE/CVF International Conference on Computer Vision. pp. 5807--5818 (2023)

\bibitem{perceiver}
Jaegle, A., Gimeno, F., Brock, A., Zisserman, A., Vinyals, O., Carreira, J.: {Perceiver: General Perception with Iterative Attention}. In: International Conference on Machine Learning. pp. 4651--4664 (2021)

\bibitem{jaynes1957information}
Jaynes, E.T.: {Information Theory and Statistical Mechanics}. Physical Review  \textbf{106}(4),  620--630 (1957)

\bibitem{sam}
Kirillov, A., Mintun, E., Ravi, N., Mao, H., Rolland, C., Gustafson, L., Xiao, T., Whitehead, S., Berg, A.C., Lo, W.Y., Doll{\'a}r, P., Girshick, R.: {Segment Anything}. In: Proceedings of the IEEE/CVF International Conference on Computer Vision. pp. 4015--4026 (2023)

\bibitem{kobyzev2025integral}
Kobyzev, I., Ghaddar, A., Hu, D., Chen, B.: {Integral Transformer: Denoising Attention, Not Too Much Not Too Little}. In: Proceedings of the Conference on Empirical Methods in Natural Language Processing. pp. 2337--2354 (2025)

\bibitem{koohpayegani2024sima}
Koohpayegani, S.A., Pirsiavash, H.: {SimA: Simple Softmax-free Attention for Vision Transformers}. In: Proceedings of the IEEE/CVF Winter Conference on Applications of Computer Vision. pp. 2607--2617 (2024)

\bibitem{kovaleva2019BERT}
Kovaleva, O., Romanov, A., Rogers, A., Rumshisky, A.: {Revealing the Dark Secrets of BERT}. In: Proceedings of the Conference on Empirical Methods in Natural Language Processing and the International Joint Conference on Natural Language Processing. pp. 4365--4374 (2019)

\bibitem{stanfordcars}
Krause, J., Deng, J., Stark, M., Fei-Fei, L.: {Collecting a Large-Scale Dataset of Fine-Grained Cars}. In: IEEE Workshop on 3D Representation and Recognition. Sydney, Australia (2013)

\bibitem{krizhevsky2009cifar10_100}
Krizhevsky, A.: {Learning Multiple Layers of Features from Tiny Images}. Tech. rep., University of Toronto (2009)

\bibitem{laplante2018comprehensive}
Laplante, P.A. (ed.): {Comprehensive Dictionary of Electrical Engineering}. CRC Press, 2 edn. (2018)

\bibitem{idefics2}
Lauren\c{c}on, H., Tronchon, L., Sanh, V.: {What matters when building vision-language models?} In: Advances in Neural Information Processing Systems. pp. 87874--87907 (2024)

\bibitem{bart}
Lewis, M., Liu, Y., Goyal, N., Ghazvininejad, M., Mohamed, A., Levy, O., Stoyanov, V., Zettlemoyer, L.: {BART: Denoising Sequence-to-Sequence Pre-training for Natural Language Generation, Translation, and Comprehension}. In: Annual Meeting of the Association for Computational Linguistics. pp. 7871--7880 (2020)

\bibitem{blip2}
Li, J., Li, D., Savarese, S., Hoi, S.: {{BLIP-2:} Bootstrapping Language-Image Pre-training with Frozen Image Encoders and Large Language Models}. In: International Conference on Machine Learning. pp. 19730--19742 (2023)

\bibitem{li2022exploring}
Li, Y., Mao, H., Girshick, R., He, K.: {Exploring Plain Vision Transformer Backbones for Object Detection}. In: Proceedings of the European Conference on Computer Vision. pp. 280--296. Springer (2022)

\bibitem{fpn}
Lin, T.Y., Doll{\'a}r, P., Girshick, R., He, K., Hariharan, B., Belongie, S.: {Feature Pyramid Networks for Object Detection}. In: Proceedings of the IEEE/CVF Conference on Computer Vision and Pattern Recognition. pp. 2117--2125 (2017)

\bibitem{mscoco}
Lin, T.Y., Maire, M., Belongie, S., Hays, J., Perona, P., Ramanan, D., Doll{\'a}r, P., Zitnick, C.L.: {Microsoft COCO: Common Objects in Context}. In: Proceedings of the European Conference on Computer Vision. pp. 740--755 (2014)

\bibitem{liu2021swin}
Liu, Z., Lin, Y., Cao, Y., Hu, H., Wei, Y., Zhang, Z., Lin, S., Guo, B.: {Swin Transformer: Hierarchical Vision Transformer using Shifted Windows}. In: Proceedings of the IEEE/CVF International Conference on Computer Vision. pp. 10012--10022 (2021)

\bibitem{liu2022videoswin}
Liu, Z., Ning, J., Cao, Y., Wei, Y., Zhang, Z., Lin, S., Hu, H.: {Video Swin Transformer}. In: Proceedings of the IEEE/CVF Conference on Computer Vision and Pattern Recognition. pp. 3202--3211 (2022)

\bibitem{lu2021soft}
Lu, J., Yao, J., Zhang, J., Zhu, X., Xu, H., Gao, W., Xu, C., Xiang, T., Zhang, L.: {SOFT: Softmax-free Transformer with Linear Complexity}. In: Advances in Neural Information Processing Systems. pp. 21297--21309 (2021)

\bibitem{cogformer}
Lv, A., Xie, R., Li, S., Liao, J., Sun, X., Kang, Z., Wang, D., Yan, R.: {More Expressive Attention with Negative Weights}. arXiv preprint arXiv:2411.07176  (2024)

\bibitem{ma2022close}
Ma, X., Wang, H., Qin, C., Li, K., Zhao, X., Fu, J., Fu, Y.: {A Close Look at Spatial Modeling: From Attention to Convolution}. arXiv preprint arXiv:2212.12552  (2022)

\bibitem{maity2025karat}
Maity, S., Hitsman, K., Li, X., Dutta, A.: {Kolmogorov-Arnold Attention: Is Learnable Attention Better For Vision Transformers?} arXiv preprint arXiv:2503.10632v2  (2025)

\bibitem{wang2025polaformer}
Meng, W., Luo, Y., Li, X., Jiang, D., Zhang, Z.: {PolaFormer: Polarity-aware Linear Attention for Vision Transformersn}. In: International Conference on Learning Representations (2025)

\bibitem{Llama3}
{Meta AI}: {Introducing Meta Llama 3: The most capable openly available LLM to date}. https://ai.meta.com/blog/meta-llama-3/ (2024), {Accessed: 2026-06-28}

\bibitem{llama32vision}
{Meta AI}: Llama 3.2: Vision and edge models. https://ai.meta.com/blog/llama-3-2-connect-2024-vision-edge-mobile-devices/ (2024), {Accessed: 2026-06-28}

\bibitem{nair2026softmax}
Nair, P.: {Softmax is $1/2$-Lipschitz: A Tight Bound Across All $\ell_p$ Norms}. Transactions on Machine Learning Research  (2026)

\bibitem{nguyen2022improving-icml}
Nguyen, T., Nguyen, T., Do, H., Nguyen, K., Saragadam, V., Pham, M., Nguyen, K.D., Ho, N., Osher, S.: {Improving Transformer with an Admixture of Attention Heads}. In: Advances in Neural Information Processing Systems. pp. 27937--27952 (2022)

\bibitem{nguyen2022fourierformer}
Nguyen, T., Pham, M., Nguyen, T., Nguyen, K., Osher, S., Ho, N.: {FourierFormer: Transformer Meets Generalized Fourier Integral Theorem}. In: Advances in Neural Information Processing Systems. pp. 29319--29335 (2022)

\bibitem{nguyen2021fmmformer}
Nguyen, T., Suliafu, V., Osher, S., Chen, L., Wang, B.: {FMMformer: Efficient and Flexible Transformer via Decomposed Near-field and Far-field attention}. In: Advances in Neural Information Processing Systems. pp. 29449--29463 (2021)

\bibitem{nguyen2023probabilistic}
Nguyen, T.M., Nguyen, T., Bui, L., Do, H., Nguyen, D.K., Le, D.D., Tran-The, H., Ho, N., Osher, S.J., Baraniuk, R.G.: {A Probabilistic Framework for Pruning Transformers Via a Finite Admixture of Keys}. In: IEEE International Conference on Acoustics, Speech and Signal Processing. pp.~1--5 (2023)

\bibitem{oquab2024dinov2}
Oquab, M., Darcet, T., Moutakanni, T., Vo, H.V., Szafraniec, M., Khalidov, V., Fernandez, P., HAZIZA, D., Massa, F., El-Nouby, A., Assran, M., Ballas, N., Galuba, W., Howes, R., Huang, P.Y., Li, S.W., Misra, I., Rabbat, M., Sharma, V., Synnaeve, G., Xu, H., Jegou, H., Mairal, J., Labatut, P., Joulin, A., Bojanowski, P.: {DINO}v2: Learning robust visual features without supervision. Transactions on Machine Learning Research  (2024)

\bibitem{clip}
Radford, A., Kim, J.W., Hallacy, C., Ramesh, A., Goh, G., Agarwal, S., Sastry, G., Askell, A., Mishkin, P., Clark, J., et~al.: {Learning Transferable Visual Models From Natural Language Supervision}. In: International Conference on Machine Learning. pp. 8748--8763 (2021)

\bibitem{raffel2020exploring}
Raffel, C., Shazeer, N., Roberts, A., Lee, K., Narang, S., Matena, M., Zhou, Y., Li, W., Liu, P.J.: {Exploring the Limits of Transfer Learning with a Unified Text-to-Text Transformer}. Journal of Machine Learning Research  \textbf{21}(140),  1--67 (2020)

\bibitem{fibottention}
Rahimian, A.K., Govind, M.K., Maity, S., Reilly, D., K{\"u}mmerle, C., Das, S., Dutta, A.: {Fibottention: Inceptive Visual Representation Learning with Diverse Attention Across Heads}. arXiv preprint arXiv:2406.19391  (2024)

\bibitem{reilly2025egoexo}
Reilly, D., Govind, M.K., Xue, L., Das, S.: {From My View to Yours: Learning Egocentric Cues from Exocentric Video using Privileged Egocentric Supervision}. In: Proceedings of the European Conference on Computer Vision (2026)

\bibitem{viscop}
Reilly, D., Govind, M.K., Xue, L., Das, S.: {VisCoP: Visual Probing for Video Domain Adaptation of Vision Language Models}. In: Proceedings of the European Conference on Computer Vision (2026)

\bibitem{saratchandran2024rethinking}
Saratchandran, H., Zheng, J., Ji, Y., Zhang, W., Lucey, S.: {Rethinking Attention: Polynomial Alternatives to Softmax in Transformers}. arXiv preprint arXiv:2410.18613  (2024)

\bibitem{accesspathsql}
Selinger, P.G., Astrahan, M.M., Chamberlin, D.D., Lorie, R.A., Price, T.G.: {Access Path Selection in a Relational Database Management System}. In: Proceedings of the ACM SIGMOD International Conference on Management of Data. pp. 23–--34 (1979)

\bibitem{shahroudy2016ntu}
Shahroudy, A., Liu, J., Ng, T.T., Wang, G.: {NTU RGB+D: A Large Scale Dataset for 3D Human Activity Analysist}. In: Proceedings of the IEEE/CVF Conference on Computer Vision and Pattern Recognition. pp. 1010--1019 (2016)

\bibitem{shi2021sparsebert}
Shi, H., Gao, J., Ren, X., Xu, H., Liang, X., Li, Z., Kwok, J.T.Y.: {SparseBERT: Rethinking the Importance Analysis in Self-attention}. In: International Conference on Machine Learning. pp. 9547--9557 (2021)

\bibitem{shuttleworth2025lora}
Shuttleworth, R., Andreas, J., Torralba, A., Sharma, P.: {LoRA vs Full Fine-tuning: An Illusion of Equivalence}. In: Advances in Neural Information Processing Systems. pp. 174627--174662 (2025)

\bibitem{tay2020efficient}
Tay, Y., Dehghani, M., Bahri, D., Metzler, D.: {Efficient Transformers: A Survey}. Association for Computing Machinery Computing Surveys  \textbf{55}(6),  1--28 (2022)

\bibitem{team2023gemini}
Team, G., Anil, R., Borgeaud, S., Wu, Y., Alayrac, J.B., Yu, J., Soricut, R., Schalkwyk, J., Dai, A.M., Hauth, A., et~al.: {Gemini: A Family of Highly Capable Multimodal Models}. arXiv preprint arXiv:2312.11805  (2023)

\bibitem{deit}
Touvron, H., Cord, M., Douze, M., Massa, F., Sablayrolles, A., Jégou, H.: {Training data-efficient image transformers \& distillation through attention}. In: International Conference on Machine Learning. pp. 10347--10357 (2021)

\bibitem{touvron2021Deep}
Touvron, H., Cord, M., Sablayrolles, A., Synnaeve, G., J{\'e}gou, H.: {Going Deeper With Image Transformers}. In: Proceedings of the IEEE/CVF International Conference on Computer Vision. pp. 32--42 (2021)

\bibitem{vaswani2017attn}
Vaswani, A., Shazeer, N., Parmar, N., Uszkoreit, J., Jones, L., Gomez, A.N., Łukasz Kaiser, Polosukhin, I.: {Attention Is All You Need}. In: Advances in Neural Information Processing Systems (2017)

\bibitem{wang2020linformer}
Wang, S., Li, B., Khabsa, M., Fang, H., Ma, H.: Linformer: Self-attention with linear complexity. arXiv preprint arXiv:2006.04768  (2020)

\bibitem{xu2019understanding}
Xu, J., Sun, X., Zhang, Z., Zhao, G., Lin, J.: {Understanding and Improving Layer Normalization}. In: Advances in Neural Information Processing Systems (2019)

\bibitem{minicpmv}
Yao, Y., Yu, T., Zhang, A., Wang, C., Cui, J., Zhu, H., Cai, T., Chen, C., Li, H., Zhao, W., He, Z., Chen, Q., Zhou, R., Zou, Z., Zhang, H., Hu, S., Zheng, Z., Zhou, J., Cai, J., Han, X., Zeng, G., Li, D., Liu, Z., Sun, M.: {Efficient {GPT-4V} level multimodal large language model for deployment on edge devices}. Nature Communications  \textbf{16}(1), ~5509 (2025)

\bibitem{mplugowl3}
Ye, J., Xu, H., Liu, H., Hu, A., Yan, M., Qian, Q., Zhang, J., Huang, F., Zhou, J.: {mPLUG-Owl3: Towards Long Image-Sequence Understanding in Multi-Modal Large Language Models}. In: International Conference on Learning Representations (2025)

\bibitem{yedifferential}
Ye, T., Dong, L., Xia, Y., Sun, Y., Zhu, Y., Huang, G., Wei, F.: {Differential Transformer}. In: International Conference on Learning Representations (2025)

\bibitem{yun2020n}
Yun, C., Chang, Y.W., Bhojanapalli, S., Rawat, A.S., Reddi, S., Kumar, S.: {$O(n)$ Connections are Expressive Enough: Universal Approximability of Sparse Transformers}. In: Advances in Neural Information Processing Systems. pp. 13783--13794 (2020)

\bibitem{zaheer2020big}
Zaheer, M., Guruganesh, G., Dubey, K.A., Ainslie, J., Alberti, C., Ontanon, S., Pham, P., Ravula, A., Wang, Q., Yang, L., Ahmed, A.: {Big Bird: Transformers for Longer Sequences}. In: Advances in Neural Information Processing Systems. pp. 17283--17297 (2020)

\bibitem{zhai2023stabilizing}
Zhai, S., Likhomanenko, T., Littwin, E., Busbridge, D., Ramapuram, J., Zhang, Y., Gu, J., Susskind, J.M.: {Stabilizing Transformer Training by Preventing Attention Entropy Collapse}. In: International Conference on Machine Learning. pp. 40770--40803 (2023)

\bibitem{Zhai_2022_CVPR}
Zhai, X., Kolesnikov, A., Houlsby, N., Beyer, L.: {Scaling Vision Transformers}. In: Proceedings of the IEEE/CVF Conference on Computer Vision and Pattern Recognition (2022)

\bibitem{vllama3}
Zhang, B., Li, K., Cheng, Z., Hu, Z., Yuan, Y., Chen, G., Leng, S., Jiang, Y., Zhang, H., Li, X., Jin, P., Zhang, W., Wang, F., Bing, L., Zhao, D.: {VideoLLaMA 3: Frontier Multimodal Foundation Models for Image and Video Understanding}. arXiv preprint arXiv:2501.13106  (2025)

\bibitem{zhang2023videollama}
Zhang, H., Li, X., Bing, L.: {Video-{LL}a{MA}: An Instruction-tuned Audio-Visual Language Model for Video Understanding}. In: Proceedings of the Conference on Empirical Methods in Natural Language Processing: System Demonstrations. pp. 543--553 (2023)

\bibitem{zhang2022neural}
Zhang, J., Li, H., Sra, S., Jadbabaie, A.: {Neural Network Weights Do Not Converge to Stationary Points: An Invariant Measure Perspective}. In: International Conference on Machine Learning. pp. 26330--26346 (2022)

\bibitem{zhang2021multi}
Zhang, P., Dai, X., Yang, J., Xiao, B., Yuan, L., Zhang, L., Gao, J.: {Multi-Scale Vision Longformer: A New Vision Transformer for High-Resolution Image Encoding}. In: Proceedings of the IEEE/CVF International Conference on Computer Vision. pp. 2998--3008 (2021)

\bibitem{apollo}
Zohar, O., Wang, X., Dubois, Y., Mehta, N., Xiao, T., Hansen-Estruch, P., Yu, L., Wang, X., Juefei-Xu, F., Zhang, N., Yeung-Levy, S., Xia, X.: {Apollo: An Exploration of Video Understanding in Large Multimodal Models}. In: Proceedings of the IEEE/CVF Conference on Computer Vision and Pattern Recognition. pp. 18891--18901 (2025)

\end{thebibliography}

\renewcommand{\theHsection}{supp.\thesection}
\renewcommand{\theHsubsection}{supp.\thesubsection}

\clearpage
\appendix

\let\maketitleold\maketitle
\makeatletter
\renewcommand{\maketitle}{\author{Ron Campos\inst{1}\orcidlink{0009-0001-8712-6727},
Subhajit Maity \inst{1}\orcidlink{0000-0002-0735-8406}, 
Xin Li \inst{1}\orcidlink{0000-0003-1201-9131},
Srijan Das\inst{2}\orcidlink{0000-0002-3373-6749},
Aritra Dutta\inst{1}\orcidlink{0000-0001-6994-1659}}%
                            \titlerunning{Denoising Attention}% 
                            \authorrunning{R. Campos et al.}% 
                            \institute{University of Central Florida, USA \and University of North Carolina at Charlotte, USA}%
                            \maketitleold
                            \begin{center}{\large\bfseries Supplementary Material}\end{center}
                            }
\makeatother
\title{DnA: Denoising Attention for Visual Tasks}
\maketitle
\setcounter{section}{0}
\renewcommand{\thesection}{\Alph{section}}
\setcounter{table}{3}
\setcounter{figure}{7}
\setcounter{theorem}{3}

\smartparagraph{Organization.} We organize the Supplementary as follows. First, we discuss a few additional works that enhance attention mechanisms in NLP and vision; see~\S\ref{app:additional_rel_works}. In \S\ref{app:theoretical_results}, we discuss our theoretical results justifying the implementation and give proof of our main result (Theorem \ref{theorem:subspace_near_orthogonal}). In \S\ref{app:supp_benchmarking}, we discuss additional experimental results, including implementation details, hyperparameter tuning, computation time analysis, diverse model analyses, and training dynamics. Finally, we present qualitative visualizations of \attnname's performance compared to softmax attention for detection and segmentation tasks, and show additional visuals for image classification and egocentric video understanding.

\begin{figure*}[t]
\centering
\includegraphics[width=1.0\linewidth]{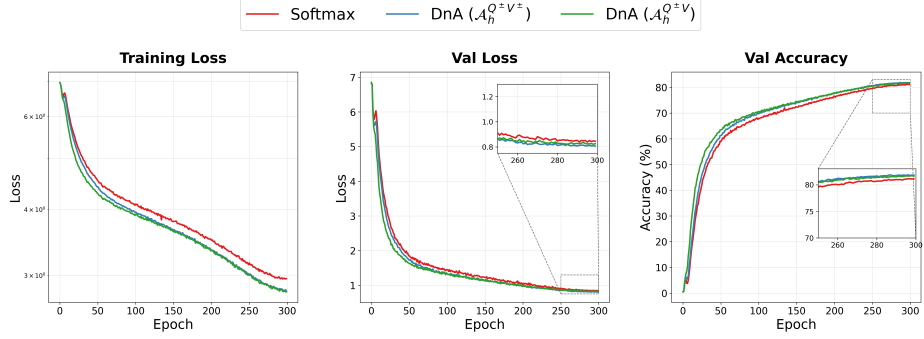}
\caption{\textbf{Training loss, validation loss, and validation accuracy of ViT-B using (left to right)} softmax, our proposed \attnname $(\mathcal{A}_h^{Q^\pm V^\pm})$, and the next best-performing design, \attnname $(\mathcal{A}_h^{Q^\pm V})$, on ImageNet-1K \cite{deng2009imnet}. \attnname $(\mathcal{A}_h^{Q^\pm V^\pm})$ displays lower training loss than ViT-B, which translates to a greater generalization performance, as shown by the consistent gap in validation loss and accuracy.}
\label{fig:supp_train_plots}
\end{figure*}

\begin{algorithm}
\caption{\attnname within a ViT block}
\label{alg:denoising attention}
\small
    \begin{algorithmic}[1]
        \State \textbf{Input:} $X \in \R^{N \times d}$
        \State \textbf{Output:} $O \in \R^{N \times d}$
        \State \textbf{Parameters:} $W_{Q^{+}},W_{Q^{-}},W_{K},W_{V^{+}},W_{V^{-}},W_O\in\R^{d \times d}, \; \alpha \in \R^{h}$
          
        \State $\gamma \gets \frac{1}{\sqrt{d/H}}$
        \State $[\alpha_1;\alpha_2;\ldots;\alpha_H] \gets \alpha$
        
        \State $Q^{+} \gets X W_{Q^{+}}$, $Q^{-} \gets X W_{Q^{-}}$
        \State $K \gets X W_{K}$
        \State $V^{+} \gets X W_{V^{+}}$, $V^{-} \gets X W_{V^{-}}$
        \State $[Q^{+}_1;Q^{+}_2;\ldots;Q^{+}_H] \gets Q^{+}$, $[Q^{-}_1;Q^{-}_2;\ldots;Q^{-}_H] \gets Q^{-}$
        \State $[K_1;K_2;\ldots;K_H] \gets K$
        \State $[V^{+}_1;V^{+}_2;\ldots;V^{+}_H] \gets V^{+}$, $[V^{-}_1;V^{-}_2;\ldots;V^{-}_H] \gets V^{-}$

        \For{each head $h \in \{1,\ldots,H\}$}
            
            \State $\cA^{Q^\pm V^\pm}_h \gets \boldsymbol{\sigma}(Q^{+}_h K_h^\top)V^{+}_h + \alpha_h~\boldsymbol{\hat{\sigma}}(\gamma Q^{-}_h K_h^\top)V^{-}_h$

        \EndFor

        \State $\hat{O} \gets [\cA^{Q^\pm V^\pm}_1;\cA^{Q^\pm V^\pm}_2;\ldots;\cA^{Q^\pm V^\pm}_H]$

        \State $O \gets \hat{O} W_O$
        \State \Return $O$
    \end{algorithmic}
\end{algorithm}

\section{Related Work --- Continued}\label{app:additional_rel_works}
Numerous works enhance attention mechanisms in NLP and vision to improve feature representation. In this scope, we mention a few of them. 

\smartparagraph{Cog attention}~\cite{cogformer} introduces negative attention weights to enhance expressiveness in language tasks. It applies softmax to the absolute values of the query-key dot product, $|p_{i,j}|$, and then scales them with their original signs, enabling both positive and negative attention weights. However, this approach suffers from slow convergence; it cannot fully replace softmax attention, requiring it in the first and last layers. Signed attention weights from the start of training are believed to cause the model to lose meaningful semantic representations, leading to optimization instability and divergence.

\smartparagraph{Polynomial activations}~\cite{saratchandran2024rethinking} are introduced as an alternative to softmax self-attention. With a constrained learnable parameter, $k$, the function $\phi(x)=\frac{x^p}{k},$ demonstrates that polynomial activations can achieve comparable performance to softmax in vision tasks. Also, see \cite{maity2025karat} for learnable attention activation for ViTs. 

\smartparagraph{Integral Transformer}~\cite{kobyzev2025integral} proposes denoising attention for language tasks by averaging signals sampled from the logit distribution. Its attention score, $\cA_I(X)=\boldsymbol{\sigma}(\frac{1}{S} \sum_{s=1}^{S} Q^s K^{s^{\top}})V$, where empirically $s=4$. A key limitation of the integral transformer is performance degradation when applied to all layers of the standard MHSA. An architecture with $\cA_I(X)$ only applied to the first half of the transformer performs better.

\smartparagraph{PolaFormer}~\cite{wang2025polaformer} addresses information loss in kernelized attention by explicitly modeling same-signed and opposite-signed query-key interactions. They divide the query and key vectors element-wise into their positive and negative components $q = q^+ - q^-$, $k = k^+ - k^-$, processing all polarity combinations through separately learned branches. Crucially, unlike softmax attention, which drowns out negative query-key interactions by exponentially mapping them to near-zero weights, PolaFormer preserves negative interactions at their original magnitudes, allowing dissimilar token relationships to contribute meaningfully to the output.

\smartparagraph{Broader Context.}
In ViTs, due to softmax, unimportant background tokens can accumulate high scores. These artifacts can degrade performance on downstream tasks such as object detection and segmentation. \emph{Register tokens}~\cite{darcet2024registers}, which are additional learnable tokens appended to the input sequence, claim to remedy this problem. These registers absorb the computational role that would otherwise fall to artifact tokens, preventing their formation. In inference, only the original patch and class tokens are used, while registers are discarded, resulting in smoother feature maps and improved qualitative performance on various visual prediction tasks.

In language processing and visual tasks, efficient attention mechanisms in terms of sparse attention~\cite{fibottention, yun2020n, shi2021sparsebert, kovaleva2019BERT, zhang2021multi, zaheer2020big, guo2019star} or linear/kernelized attention~\cite{hanbridging, nguyen2023probabilistic, nguyen2021fmmformer, nguyen2022fourierformer, lu2021soft, nguyen2022improving-icml} are proposed, which are orthogonal to our work.

\section{Theoretical Results}\label{app:theoretical_results}
This section provides a theoretical guarantee supporting the design principle of \attnname. 

In attention mechanisms, the most intuitive justification for softmax is its role in converting attention scores into a valid probability distribution. The following Theorem from \cite{vaswani2017attn, bahdanau2014neural} justifies this:
\begin{theorem}[\emph{Attention as Probability Distribution}]
\label{thm:attention_probability}
Let $S = \frac{QK^\top}{\sqrt{d_k}}$ be the raw attention scores. The softmax function ensures that:
\begin{enumerate}
\item $\boldsymbol{\sigma}(S)_{ij} \geq 0$, for all $i,j$ (non-negativity)
\item $\sum_j \boldsymbol{\sigma}(S)_{ij} = 1$, for all $i$ (normalization)
\item The output is a convex combination of value vectors: $\mathbf{o}_i = \sum_j \alpha_{ij} \mathbf{v}_j$ with $\alpha_{ij} \geq 0$ and $\sum_j \alpha_{ij} = 1$
\end{enumerate}
\end{theorem}

Next, we discuss Proposition \ref{proposition:compact}, which is originally adapted from Zhang et al.\ \cite{zhang2022neural}. This result examines the convergence of an $L$-layer deep neural network, trained by SGD with weight decay, under an invariant-measure perspective. This result conveys a key message: \emph{all iterates lie in a compact region almost surely even though the function may not be smooth or Lipschitz continuous.} 

\subsection{How Proposition \ref{proposition:compact} can be adapted for Transformers?}
Zhang et al.\ \cite{zhang2022neural} view convergence in neural network training in terms of an invariant measure. Especially, Theorem 4.3 in \cite{zhang2022neural} shows that as the training loss $\cL(\theta)$ stabilizes, the iterates (or the weights) lie in a compact region almost surely; Proposition \ref{proposition:compact} is the modification of the result of Theorem 4.3. In Theorem 4.3, Zhang et al.\ considered an $L$-layer deep neural network, whose loss, $\cL$ is $c_l$-Lipschitz with respect to input, $x$ (see Lemma 4.1 in \cite{zhang2022neural}), each activation function $\boldsymbol{\sigma}_l$ is (sub)-differentiable and $c_\nu$-Lipschitz for some constant $c_\nu>0,$ and $\boldsymbol{\sigma}_l(0)=0$ (see Assumption 4.2 in \cite{zhang2022neural}).

Proposition~\ref{proposition:compact} applies to architectures expressible as compositions of linear and Lipschitz maps under bounded inputs. Accordingly, for Transformers, we assume the embedded inputs are bounded and denote the current input sequence by $X_p$. The MLP sub-blocks share the structure of standard DNNs and therefore satisfy the assumption when their activation functions are Lipschitz~\cite{goodfellow2016deep}. Transformers additionally include residual connections and normalization. Residual connections of the form $X_p + F(X_p)$ preserve the same decomposition up to a constant~\cite{he2016resnet}. LayerNorm operates elementwise across features and, under bounded activations, is Lipschitz on the bounded set, allowing it to be treated as another Lipschitz map in the composition~\cite{xu2019understanding}.

For the multi-head attention (MHA) sub-blocks, the sequence $X_p$ is linearly projected to queries, keys, and values: $Q=X_pW_Q$, $K=X_pW_K$, and $V=X_pW_V$, with the final output projected by $W_O$. Each head computes attention weights from $QK^\top$ using softmax and masking, and applies them to $V$~\cite{vaswani2017attn}. Define the linear map $\mathcal{T}(X_p):=[X_pW_Q, X_pW_K, X_pW_V]$ and let $\boldsymbol{\sigma}(\mathcal{T}(X_p))=\boldsymbol{\sigma}\left(\frac{QK^\top}{\sqrt{d_k}}+M\right)V$, where $M$ is the attention mask. Then one attention head can be written as $\mathcal{V}(X_p)=(\boldsymbol{\sigma}_{(Q,K,V)}\circ\mathcal{T}(X_p))W_O$. We treat $\boldsymbol{\sigma}_{(Q,K,V)}$ as a Lipschitz operator with constant $L_{\boldsymbol{\sigma}}$; in particular, the softmax function is $\tfrac{1}{2}$-Lipschitz uniformly across all $\ell_p$ norms~\cite{nair2026softmax}.

\smartparagraph{Proof of the Main Result: Theorem \ref{theorem:subspace_near_orthogonal}.} Proposition \ref{proposition:compact} claims that as long as $W_{V_{ij}(0)}$ be initialized within a compact set, $\cC_V:=\{W_{V_{ij}(0)}: \|W_{V_{ij}(0)}\|\le \gamma\},$ for every $k,$ the iterate $W_{V_{ij}(k)}$ lies in $\cC_V$ with high probability. We additionally assume that if the entries of $W_{V^{+}}$ and $W_{V^{-}}$ be initialized as i.i.d. $\cN(0,\sigma^2)$ over a finite support $[-\gamma, \gamma]$, for every $k,$ the entries of the iterates $\{W_{V^{+}{(k)}}, W_{V^{-}{(k)}}\}$ are i.i.d. sub Gaussian with parameter $\sigma>0$. Then we show $|(W_{V^{+}{(k)}})_{pq}|\le \gamma$ and $|(W_{V^{-}{(k)}})_{pq}|\le \gamma$, with high probability. The boundedness helps us prove our main claim of \emph{quasi-orthogonality} of the subspaces, $V^+$ and $V^-$, in Theorem \ref{theorem:subspace_near_orthogonal}. 

\setcounter{theorem}{2}
\begin{theorem}[\textbf{Quasi-orthogonality of the subspaces}]
   Let the entries of $W_{V^{+}}$ and $W_{V^{-}}$ be initialized as i.i.d. $\cN(0,\sigma^2)$ over a finite support $[-\gamma, \gamma].$  Let the input $X_p\in\R^{N\times d}$ be such that $\|X_p\|_{\infty}\le M$. Assume for every $k,$ the entries of the iterates $\{W_{V^{+}{(k)}}, W_{V^{-}{(k)}}\}$ are i.i.d. sub Gaussian with parameter $\sigma>0$. \myNum{i} Then $|(W_{V^{+}{(k)}})_{pq}|\le \gamma$ and $|(W_{V^{-}{(k)}})_{pq}|\le \gamma$, with high probability. \myNum{ii} If $\gamma=O(M^{-1}{\epsilon}^{-\frac{1}{2}}(N\log(\delta^{-1}))^{-\frac{1}{4}})$, then $|\langle V^+,V^-\rangle_F|\le d\epsilon$ with probability $1-\delta.$
\end{theorem}

\begin{proof}
Since our Assumption holds for all $k$, for simplicity, we dispense with $k$ from our proof. We use the notation, $A_j$ to denote the $j^{\rm th}$ column of a matrix, $A$, and $(A)_{ij}$ to denote its $(i,j)^{\rm th}$ entry. 

We start by expanding $\langle V^+,V^- \rangle_F$ as 
\begin{equation*}
      |\langle V^+,V^-\rangle_F|=|{\rm Trace}(W_{V^{+}}^\top X_p^\top X_p W_{V^{-}})|.
\end{equation*}
Let $(X_pW_{V^{+}})_k$ and $(X_pW_{V^{-}})_l$ be the $k^{\rm th}$ and $l^{\rm th}$ columns of $X_pW_{V^{+}}$ and $X_p W_{V^{-}}$, respectively. Therefore, we can write
\begin{eqnarray*}
(W_{V^{+}}^\top X_p^\top X_p W_{V^{-}})_{kl}&=&(X_p W_{V^{+}})_k^\top(X_p W_{V^{-}})_l\\
&=&\sum_{i=1}^N\underbrace{(X_p W_{V^{+}})_{ik}(X_p W_{V^{-}})_{il}}_{:=Z_i}.
\end{eqnarray*}
Recall that, for every $k\in [1,2,\cdots],$ the entries of the iterates $\{W_{V^{+}{(k)}}, W_{V^{-}{(k)}}\}$ assumed to be i.i.d. sub Gaussian with and mean 0, and parameter $\sigma>0$, and by Proposition 1, we have $|(W_{V^{+}{(k)}})_{pq}|\le \gamma$ and $|(W_{V^{-}{(k)}})_{pq}|\le \gamma$, with probability $1-\delta_\gamma-\Delta (L,\cdots),$ where, $\delta_\gamma=e^{-\frac{\gamma^2}{2\sigma^2}}$ and $\Delta (L,\cdots)$ depends on the number of layers ($L$), Lipschitz constants of the activations.   

Finally, the terms $Z_i$ are i.i.d. with $\mathbb{E}[Z_i]=\mathbb{E}[(X_p W_{V^{+}})_{ik}(X_pW_{V^{-}})_{il}]=0.$ Additionally, as $\|X_p\|_{\infty}\le M$, we have $|Z_i|\leq M^2\gamma^2.$ Now, we use Hoeffding's inequality to $\bigl\vert\sum_{i=1}^N Z_i\bigr\vert$, where $-M^2\gamma^2 \leq Z_i\leq M^2\gamma^2$, and obtain:
\begin{eqnarray*}
\mathbb{P}\left(\bigl\vert (W_{V^{+}}^\top X_p^\top X_p W_{V^{-}})_{kl}\bigr\vert\geq \epsilon \right)&=&\mathbb{P}\left(\bigl\vert\sum_{i=1}^NZ_i\bigr\vert \geq \epsilon\right)\\
&\overset{\mathbb{E}\left[\sum_{i=1}^NZ_i\right]=0}{=}&\mathbb{P}\left(\bigl\vert\sum_{i=1}^NZ_i -\mathbb{E}\left[\sum_{i=1}^NZ_i\right]\bigr\vert \geq \epsilon\right)\\
&\overset{\rm Hoeffding's\;inequality}{\leq}& 2{\rm exp}\left(\frac{-2\epsilon^2}{\sum_{i=1}^N(2\gamma^2M^2)^2}\right)\\
&=&2{\rm exp}\left(\frac{-\epsilon^2}{2N\gamma^4M^4}\right).
\end{eqnarray*}
By setting $\delta = 2{\rm exp}\left(\frac{-\epsilon^2}{2N\gamma^4M^4}\right),$ we have $\epsilon=\gamma^2M^2\sqrt{2N\log(\frac{2}{\delta})}.$ Therefore, with probability at least $1-\delta$, we have $|(W_{V^{+}}^\top X_p^\top X_p W_{V^{-}})_{kl}|\leq \epsilon.$ Thus, with probability at least $1-\delta$, we have
\begin{equation*}
      |\langle V^+,V^-\rangle_F|\leq d\epsilon = d\gamma^2M^2\sqrt{2N\log\left(\frac{2}{\delta}\right)}.
\end{equation*}
Therefore, if $\gamma=O(M^{-1}{\epsilon}^{-\frac{1}{2}}(N\log(\delta^{-1}))^{-\frac{1}{4}})$, then $|\langle V^+,V^-\rangle_F|\le d\epsilon$ with probability at least $1-\delta.$ Hence, the result. 
\end{proof}

\section{Addendum to the Benchmarking and Evaluation}\label{app:supp_benchmarking}
This section complements our empirical results in \S\ref{sec:expt}. In addition, we present and analyze results for other downstream tasks such as \myNum{i} \emph{transfer learning}, \myNum{ii} \emph{object detection and instance segmentation}, and \myNum{iii} \emph{robustness benchmarks}, followed by further analysis of \attnname.

\subsection{Implementation Details}
\label{app:supp_implementation}
We outline the implementation details for the proposed \attnname. The pseudocode for \attnname is given in \Cref{alg:denoising attention}.

\smartparagraph{Image Classification.}\label{app:supp_impl_classification}
\textit{Image classification} refers to the pretraining experiments performed on ImageNet-1K~\cite{deng2009imnet}; see \S\ref{sec:expt}. We train all baselines and our \attnname variants for 300 epochs with an effective global batch size of 1024. We use gradient accumulation across 2 iterations with a local batch size of 512 for all the models except the \attnname variant using attention $\cA^{Q^\pm V^\pm}_h$, which uses a local batch size of 256 with gradient accumulation across 4 iterations due to GPU memory constraints; see Tables \ref{tab:imagenet1k} and \ref{tab:ablation_qv} for results. We additionally evaluate \attnname across other architectures, including ViT-T~\cite{vit,deit}, Swin-B~\cite{liu2021swin}, and ConViT-B~\cite{dascoli2021convit}; see Table~\ref{tab:accuracy_comparison}. See \Cref{tab:supp_imnet_hp} for detailed hyperparameter setup on ViT-B.

\begin{table}
\centering
\footnotesize
\setlength{\tabcolsep}{8pt}
\begin{tabular}{lr}
\toprule
Input Size & $224\times224$ \\
Patch Size & 16 \\
Crop Ratio & $0.9$ \\
Batch Size & $1024$\\
\midrule
Optimizer & AdamW \\
Optimizer Epsilon & $10^{-8}$ \\
Weight Decay & $0.075$ \\ 
Stochastic Depth & $0.15$ \\
\midrule
Learning Rate Schedule & Cosine\\
Learning Rate & $10^{-3}$\\
Warmup LR & $10^{-6}$ \\
Min LR & $10^{-5}$ \\
Epochs & 300 \\
Warmup Epochs & 5 \\
\midrule
Random Resize \& Crop Scale and Ratio & $(0.08, 1.0), (0.67, 1.5)$\\
Color Jittering & $0.3$ \\
Rand Augment & 9/0.5\\
Mixup & $0.8$ \\
Cutmix & $1.0$ \\
Mixup Mode & Batch \\
Label Smoothing & $0.1$ \\
\bottomrule
\end{tabular}
\caption{{Hyperparameter setup for the ViT-B pretraining on ImageNet-1K~\cite{deng2009imnet}.}
}
\label{tab:supp_imnet_hp}
\end{table}

\myNum{i}\smartparagraph{Transfer Learning.} For the downstream image classification tasks on CIFAR-10, CIFAR-100~\cite{krizhevsky2009cifar10_100}, and Stanford Cars~\cite{stanfordcars}, we use the same hyperparameters, except for 100 training epochs, a weight decay of $0.05$, and a stochastic drop of $0.1$.

\myNum{ii}\smartparagraph{Object Detection \& Instance Segmentation.}\label{app:supp_impl_detseg}
We use the ViT-Det~\cite{li2022exploring} baseline, which employs a Mask R-CNN-based~\cite{he2017mask} framework with a feature pyramid network~\cite{fpn}. We initialize the ViT-Det backbones with the weights of ~\S\ref{sec:expt_image} for object detection and instance segmentation on MS-COCO~\cite{mscoco}, and train the models for 100 epochs (equivalent to 184,375 iterations with a batch size of 32); see the hyperparameter setup in \Cref{tab:supp_coco_hp}.

\myNum{iii}\smartparagraph{Robustness Bechmarks.} In addition to the ImageNet-A~\cite{imagenetao} benchmarking provided in Table~\ref{tab:imageneta}, we also evaluate the robustness of \attnname on ImageNet-R~\cite{imagenetr} and ImageNet-O~\cite{imagenetao}. ImageNet-O contains images that are similar to ImageNet-1K but are out of distribution. ImageNet-R contains images that are different renditions of the standard 200 classes from ImageNet-1K, such that the images are out of domain and out of distribution. For these datasets, we use calibration error (RMSCE), area under the curve (AUROC), and reliability (AURRA) metrics, along with the accuracy.

\begin{table}[t]
\centering
\footnotesize
\setlength{\tabcolsep}{2.5pt}
\begin{tabular}{lr}
\toprule
Input Size & $224\times224$ \\
Batch Size & 32\\
\midrule
Optimizer & AdamW \\
Optimizer Epsilon & $10^{-6}$ \\
Momentum & $0.9, 0.999$ \\
Weight Decay & $0.1$ \\ 
\midrule
Learning Rate Schedule & Warmup Scheduler\\
Learning Rate & $10^{-4}$ \\
Iterations & 184375 \\
Decay Rate & $0.988$ \\
\midrule
Horizontal Random Flip & $0.5$ \\
Vertical Random Flip & $0.0$ \\
\midrule
Patch Size & 16 \\
Attention Window Size & 0 \\
Encoder Output Layer Index & $11$ \\
\midrule
Pyramid Scale Factors & $4.0, 2.0, 1.0, 0.5$ \\
Output Channels & 256 \\
\midrule
Proposal Generator Input Layers & $p2, p3, p4, p5, p6$ \\
Proposal Generator Input Sizes & $32, 64, 128, 256, 512$ \\
Proposal Generator Training Pre-NMS Top-K & $2000$ \\
Proposal Generator Evaluation Pre-NMS Top-K & $1000$ \\
Proposal Generator Training Post-NMS Top-K & $1000$ \\
Proposal Generator Evaluation Post-NMS Top-K & $1000$ \\
Proposal Generator NMS Threshold & $0.7$ \\
\midrule
Region-of-Interest Heads IOU Threshold & $0.5$ \\
Region-of-Interest Score Threshold & $0.05$ \\
Region-of-Interest NMS Threshold & $0.5$ \\
\bottomrule
\end{tabular}
\caption{{Hyperparameter setup for the ViT-Det training on MS-COCO~\cite{mscoco} for object detection and instance segmentation.}
}
\label{tab:supp_coco_hp}
\end{table}

\subsection{Training Details}\label{app:supp_training_details}
We examine the training dynamics and computational requirements of our implementation, focusing on the ViT-B model from \S\ref{sec:expt_image}.

\smartparagraph{Loss and Accuracy Curves.} Figure~\ref{fig:supp_train_plots} displays the training and validation losses, as well as the validation accuracy curves for ViT-B models using softmax, our proposed \attnname, $\mathcal{A}_h^{Q^\pm V^\pm}$, and the next best-performing design, $\cA_h^{Q^\pm V}$; see~\S\ref{app:other dna}. Throughout training, \attnname consistently achieves lower training and validation loss compared to the standard softmax implementation, with $\cA_h^{Q^\pm V}$ following a similar trend. The improvement in the loss is reflected by the higher validation accuracy and improved final convergence.

\smartparagraph{Computational Requirements.} Different denoising designs vary in parameter count, depending on the query and value configuration. This results in different FLOPs counts and GPU memory (VRAM) requirements. However, their overall training time remains comparable under the same hardware configuration, and the throughput also remains similar; see \Cref{tab:supp_computation}.

\begin{table*}[t]
\centering
\scriptsize
\setlength{\tabcolsep}{2.5pt}
\begin{tabular}{lcccccc}
\toprule
\multirow{2}{*}{\textbf{Model}} & \multicolumn{2}{c}{\textbf{Validation Set}} & \multicolumn{2}{c}{\textbf{Test Set}} & \multirow{2}{*}{\textbf{Parameters}} & \multirow{2}{*}{\textbf{GFLOPs}}\\
\cmidrule(lr){2-3}\cmidrule(lr){4-5}
& \textbf{Acc.@1} & \textbf{Acc.@5} & \textbf{Acc.@1} & \textbf{Acc.@5} & & \\
\midrule
ViT-B & 81.1$\pm$0.007 & 95.5$\pm$0.002 & 81.0$\pm$0.009 & 95.6$\pm$0.002 & 86.6M & 17.6\\
\midrule
+Diff. Attn. & 81.5$\pm$0.009 & 95.6$\pm$0.002 & 81.6$\pm$0.016 & 95.7$\pm$0.007 & 86.6M & 17.6\\
+Cog Attn. & 81.5$\pm$0.009 & 95.7$\pm$0.000 & 81.5$\pm$0.002 & 95.7$\pm$0.002 & 86.6M & 17.6\\
\midrule
\rowcolor{Gray}
\textbf{+DnA} & \textbf{81.9}$\pm$0.002 & \textbf{95.8}$\pm$0.002 & \textbf{81.8}$\pm$0.002 & \textbf{95.8}$\pm$0.002 & 100.7M & 21.1\\
\bottomrule
\end{tabular}
\caption{{Accuracy on the ImageNet-1K~\cite{deng2009imnet} validation and test set, averaged over 3 independent runs.}}
\label{tab:imagenet1k_var}
\end{table*}

\begin{table*}[t]
\centering
\begin{tabular}{lccccc}
\toprule
\textbf{Model}                               & \textbf{Parameters} & \textbf{GFLOPs} & \textbf{VRAM}            & \textbf{Training Time} & \textbf{Throughput} \\
\midrule
ViT-B \cite{vit,deit}                              & 86.6M      & 17.6   & 56 GB         & 110.25 h      & 909.1 img/s  \\
\midrule
+\attnname ($\cA_h^{\pm}$)               & 86.6M      & 17.6   & 70 GB         & 112.98 h      & 909.1 img/s  \\
+\attnname ($\hat{\cA}_h^{\pm}$)         & 86.6M      & 17.6   & 71 GB         & 114.25 h      & 909.1 img/s  \\
+\attnname ($\cA_h^{Q^\pm V}$)           & 93.6M      & 19.3   & 76 GB         & 111.38 h      & 909.1 img/s  \\
+\attnname ($\cA_h^{QV^\pm}$)            & 93.6M      & 19.3   & 76 GB         & 114.25 h      & 909.1 img/s  \\
\rowcolor{Gray}
+\attnname ($\cA_h^{Q^\pm V^\pm}$)       & 100.7M     & 21.1   & 45 GB         & 117.08 h      & 909.1 img/s  \\
\bottomrule
\end{tabular}
\caption{{Detailed computational requirements for all the \attnname variants compared to the baseline ViT-B, which uses softmax attention. We note an increase in parameters, computations (GFLOPs), and GPU memory (VRAM) requirements while training time remains comparable (changes are modest), and the throughput (inference speed) remains the same.}
}
\label{tab:supp_computation}
\end{table*}

\subsection{Evaluation on Downstream Tasks}\label{app:additional_benchmarking}
In this section, we show the performance of \attnname in various downstream tasks.

\myNum{i}\smartparagraph{Transfer learning.} 
\Cref{tab:transfer_learning} shows transfer learning results on CIFAR-10, CIFAR-100, and Stanford Cars test sets. While the CIFAR-10 test accuracy results match those of the baseline, we achieve a 0.7\% gain on CIFAR-100 and a 0.4\% gain on Stanford Cars.

\myNum{ii}\smartparagraph{Object Detection and Instance Segmentation.} 
\Cref{tab:coco_box_mask} shows that \attnname outperforms the vanilla ViT-Det baseline on both object detection (Box) and instance segmentation (Mask) tasks. Specifically, our approach achieves improvements of +0.3 AP, +0.6 AP$_{50}$, and +0.6 AP$_{75}$ for box detection, and +0.1 AP, +0.5 AP$_{50}$ for mask segmentation, demonstrating that \attnname transfers effectively to downstream detection and segmentation tasks. We show qualitative results in \S\ref{app:supp_detseg}.

\myNum{iii}\smartparagraph{Robustness Evaluations.} To understand the robustness and tolerance of the proposed model, we perform benchmarking with ImageNet-O and ImageNet-R. From \Cref{tab:imagenet_var_results}, we observe that \attnname outperforms vanilla ViT-B in almost all metrics across both benchmarks. On ImageNet-O, it improves AUROC by 0.4\% and AUPR by 0.9\%, while on ImageNet-R, it achieves a 1.3\% accuracy gain and an AURRA gain of 0.8\%.

\begin{table}[t]
    \centering
    \setlength{\tabcolsep}{5pt}
    \begin{tabular}{lcccccc}
    \toprule
    \multirow{2}{*}{\textbf{Model}}   & \multicolumn{2}{c}{\textbf{CIFAR-10~\cite{krizhevsky2009cifar10_100}}} & \multicolumn{2}{c}{\textbf{CIFAR-100~\cite{krizhevsky2009cifar10_100}}} & \multicolumn{2}{c}{\textbf{Cars~\cite{stanfordcars}}} \\
    \cmidrule(lr){2-3}\cmidrule(lr){4-5}\cmidrule(lr){6-7}
                                      & \textbf{Acc.@1} & \textbf{Acc.@5} & \textbf{Acc.@1} & \textbf{Acc.@5} & \textbf{Acc.@1} & \textbf{Acc.@5} \\
    \midrule
    ViT-B                             & 98.6            & 99.9            & 88.8            & 98.0            & 91.5            & 99.0            \\
    \rowcolor{Gray}
    \textbf{+\attnname ($\cA_h^{Q^\pm V^\pm}$)}& 98.6            & 99.9            & \textbf{89.5}            & \textbf{98.1}            & \textbf{91.9}            & 99.0            \\
    \bottomrule
    \end{tabular}
    \caption{{Quantitative comparison of \attnname and the baseline ViT-B for the image classification task in a transfer learning setup across CIFAR-10, CIFAR-100~\cite{krizhevsky2009cifar10_100}, and Stanford Cars~\cite{stanfordcars} datasets. Each model is fine-tuned for 100 epochs.}}
    \label{tab:transfer_learning}
\end{table}

\begin{table}[t]
    \centering
    \setlength{\tabcolsep}{10.5pt}
    \begin{tabular}{lcccccc}
    \toprule
    \multirow{2}{*}{\textbf{Model}}    & \multicolumn{3}{c}{\textbf{Detection (Box)}}       & \multicolumn{3}{c}{\textbf{Segmentation (Mask)}}   \\
    \cmidrule(lr){2-4}\cmidrule(lr){5-7}
                                       & AP   & AP\textsubscript{50} & AP\textsubscript{75} & AP   & AP\textsubscript{50} & AP\textsubscript{75} \\
    \midrule
    ViT-Det-B                          & 26.2 & 42.4                 & 26.8                 & 22.7 & 38.9                 & 22.9 \\
    \rowcolor{Gray}
    \textbf{+\attnname ($\cA_h^{Q^\pm V^\pm}$)} & \textbf{26.5} & \textbf{43.0}                 & \textbf{27.4}                 & \textbf{22.8} & \textbf{39.4}                 & 22.9 \\
    \bottomrule
    \end{tabular}
    \caption{{Quantitative comparison of ViT-Det~\cite{li2022exploring} using softmax attention and \attnname on the MS COCO \cite{mscoco} validation set for object detection (Box) and instance segmentation (Mask). See \S\ref{app:supp_implementation} for implementation details and hyperparameter setup. \emph{Due to computational constraints, we opt for a more feasible configuration with $224\times224$ images and 50 training epochs. Moreover, our ViTDet baseline also uses supervised ViT-B pretrained weights, instead of MAE pretrained weights, for a fair comparison with \attnname.}}}
    \label{tab:coco_box_mask}
\end{table}

\begin{figure}[ht]
\thisfloatsetup{subfloatrowsep=none}
\centering
\begin{floatrow}[1]
\ffigbox[\linewidth]{
\begin{subfloatrow}[1]
\centering
    \ffigbox[\textwidth]{%
        \centering
        \includegraphics[width=0.9\linewidth]{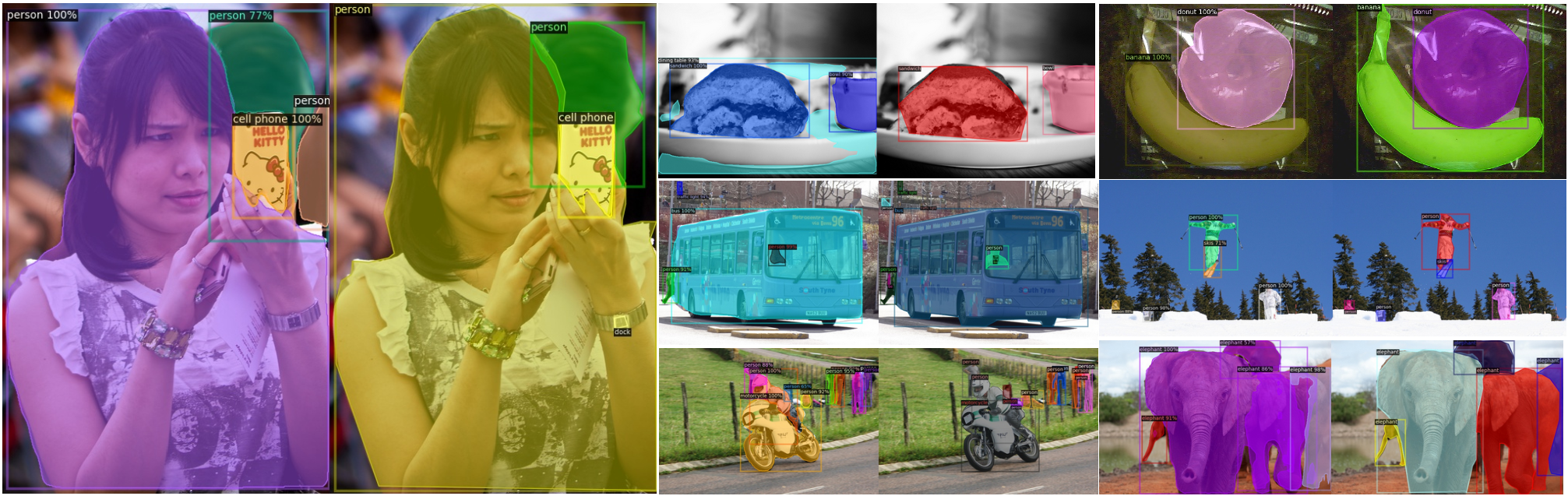}
        }{%
        \subcaption{ViT-Det with \attnname. 
        }\label{fig:adeit_detr}
        }
\end{subfloatrow}\\
\begin{subfloatrow}[1]
    \ffigbox[\textwidth]{%
        \centering
        \includegraphics[width=0.9\linewidth]{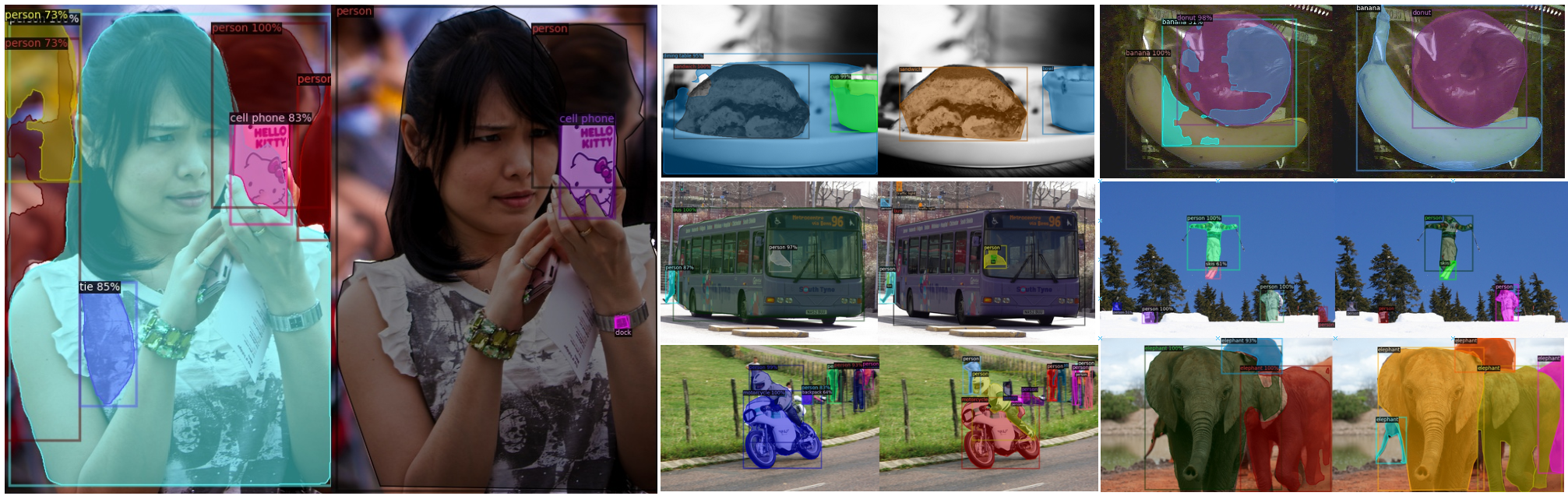}
        }{%
        \subcaption{ViT-Det with softmax attention.
        }\label{fig:softmax_detr}
        }
\end{subfloatrow}
}{
\caption{Visualization of the inference results of ViT-Det on a ViT-B backbone (a) using \attnname \textbf{(top row)} and (b) using softmax attention with the corresponding ground truths \textbf{(bottom row)}; zoom in for a better view. For each sample, the left image is the prediction and the right image is the ground truth. 
}
\label{fig:supp_detseg_vis}
}
\end{floatrow}
\end{figure}

\begin{table}[t]
\centering
\begin{tabular}{lcccccc}
\toprule
\multirow{2}{*}{\textbf{Model}} 
& \multicolumn{3}{c}{\textbf{ImageNet-O~\cite{imagenetao}}} 
& \multicolumn{3}{c}{\textbf{ImageNet-R~\cite{imagenetr}}} \\
\cmidrule(lr){2-4} \cmidrule(lr){5-7}
& \textbf{FPR95 $\downarrow$}
& \textbf{AUROC $\uparrow$} 
& \textbf{AUPR $\uparrow$} 
& \textbf{Acc.} $\uparrow$  
& \textbf{RMSCE $\downarrow$} 
& \textbf{AURRA $\uparrow$} \\
\midrule
ViT-B 
& \textbf{89.3} & 65.7 & 23.4 
& 45.2 & \textbf{5.0} &  74.7 \\
\rowcolor{Gray}
\textbf{+\attnname ($\cA_h^{Q^\pm V^\pm}$)} 
& 90.5 & \textbf{66.1} & \textbf{24.3} 
& \textbf{46.5}  & 5.3 & \textbf{75.5} \\
\bottomrule
\end{tabular}
\caption{
\small{Comparison of model performances across ImageNet-O~\cite{imagenetao} and ImageNet-R~\cite{imagenetr} benchmarks. These benchmarks indicate how robust and tolerant the models are towards degradation and distribution shift.
\label{tab:imagenet_var_results}
}
}
\end{table}

\begin{table}[h]
\centering
\begin{tabular}{l c >{\columncolor{Gray}}c >{\columncolor{Gray}}c c}
\toprule
\textbf{Model} & \textbf{Acc.@1} & \textbf{+\attnname} & \textbf{Acc.@1 (Gain)} & \textbf{Epochs} \\
\midrule
ViT-T~\cite{vit,deit}   & 58.9 & \checkmark & \textbf{61.9} (\greenup $ 3.0$) & 100 \\
Swin-B~\cite{liu2021swin} & 83.4 & \checkmark & \textbf{83.7} (\greenup$ 0.3$) & 300 \\
ConViT-B~\cite{dascoli2021convit} & 82.3 & \checkmark & \textbf{82.5} (\greenup$ 0.2$) & 300 \\
\bottomrule
\end{tabular}
% }
\caption{ImageNet-1K~\cite{deng2009imnet} accuracy of \attnname when integrated within ViT-T and other transformer-based architectures. The \greenup~arrows indicate the absolute gain compared to the base model.}
\label{tab:accuracy_comparison}
\end{table}

\subsection{Qualitative Analysis on Detection \& Segmentation}\label{app:supp_detseg}

\Cref{fig:adeit_detr} displays the qualitative performance of using \attnname in ViT-Det. \Cref{fig:softmax_detr} shows the performance of ViT-Det with traditional softmax. \attnname generates more semantically correct masks. This outlines superior object representation, which translates to more accurate detection.

\begin{figure*}
    \centering
    \includegraphics[width=0.9\linewidth]{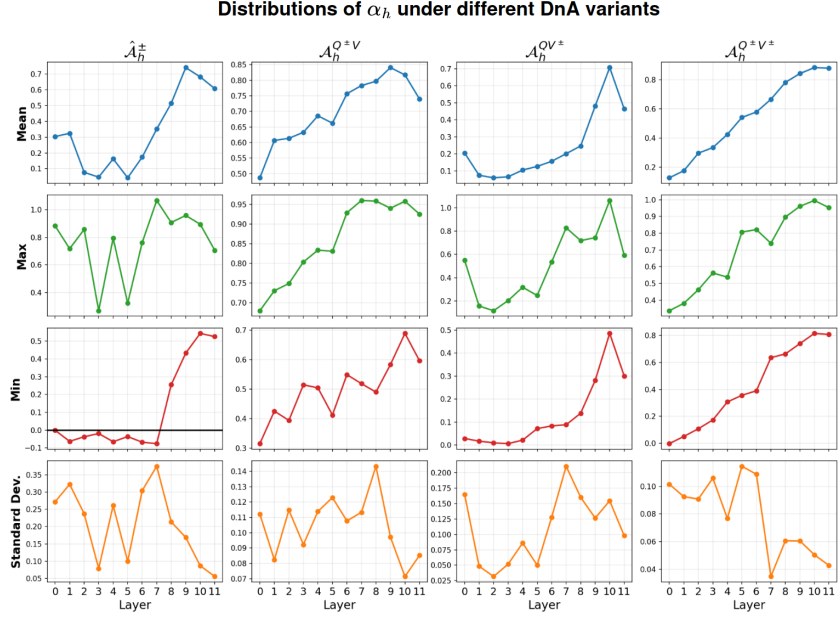}
    \caption{Statistics of $\alpha_h$ across different denoising designs. Each row represents the mean, maximum, minimum, and standard deviation of the learnable parameter, $\alpha_h$, for each layer of the network.
    }
    \label{fig:alphas}
\end{figure*}

\subsection{Extended Analysis: A Closer Look at Different Design Components and Behavior of \attnname}\label{app:supp_analysis}
The following section presents ablations, a closer look at the behavior of \attnname, and visualizations complementing the results in \S\ref{sec:expt_image} and \S\ref{sec:expt_vid_und}.

\myNum{A}\smartparagraph{Finding alternate designs with denoising effect.}\label{app:other dna}
The design of \attnname is directly motivated by our theoretical analysis and involves additional branches for query and value for complete separation of subspaces. Here, we intend to understand how these additional branches contribute to the denoising effect of the proposed \attnname and whether the alternative design choices involving these additional branches can achieve the denoising effect. This study gives us an understanding of the importance of subspace separation.

\attnname attends to the discriminative features in the attention space by enabling positive and negative token interactions in two different subspaces. Primarily, we ask how important the subspace separation is from a practical approach. Thus, we evaluate $\cA^{\pm}_h = [\boldsymbol{\sigma}(\frac{Q_hK_h^\top}{\sqrt{d}}) + \boldsymbol{\hat{\sigma}}(\frac{Q_hK_h^\top}{\sqrt{d}})]V_h$ using the same learnable triplets $(Q_h, K_h, V_h)$ for each head, but use both softmax and softmin. We also keep a version of this, $\hat{\cA}_h^{\pm} = [\boldsymbol{\sigma}(\frac{Q_hK_h^\top}{\sqrt{d}}) + \alpha_h \boldsymbol{\hat{\sigma}}(\frac{Q_hK_h^\top}{\sqrt{d}})]V_h$ that uses the learnable scalar, $\alpha_h$ per head. As another alternative to the complete separation of subspaces, we ask whether a partial separation using either the query or the value is enough. Thus, we obtain two additional variants. The first one with separate queries, but the same value: $\cA^{Q^{\pm}V}_h = \left[\boldsymbol{\sigma}\left(\frac{Q_h^{+} K_h^\top}{\sqrt{d}}\right) + \alpha_h\boldsymbol{\hat{\sigma}}\left(\frac{Q_h^{-} K_h^\top}{\sqrt{d}}\right)\right]V_h$. The second one has separate values, but the same query: $\cA^{QV^{\pm}}_h = \boldsymbol{\sigma}\left(\frac{Q_h K_h^\top}{\sqrt{d}}\right)V^{+}_h + \alpha_h\boldsymbol{\hat{\sigma}}\left(\frac{Q_h K_h^\top}{\sqrt{d}}\right)V^{-}_h$. In both these cases, $\alpha_h$ acts as the balancing parameter and remains learnable.

\smartparagraph{Every component of \attnname is necessary.} We ablated several design variations of \attnname and summarize their results in Table~\ref{tab:ablation_qv}. The results consistently back our theoretical analysis.

\myNum{a} Attention $\cA_h^{\pm}$ shares query and value spaces and performs subpar compared to ViT-B. Without subspace separation, the model cannot properly separate positive and negative interactions. \myNum{b} While $\hat{\cA}_h^{\pm}$ improves the accuracy by 2.9\% relative to that of $\cA_h^{\pm}$, it remains limited by the shared subspace. \myNum{c} Attention $\cA_h^{Q^\pm V}$, uses only two separate query spaces, and performs \emph{second-best} on ImageNet-1K. \myNum{d} $\cA_h^{QV^\pm}$ uses only separate value spaces and performs comparably to the baseline, but shows that separate value spaces alone cannot be properly utilized without separate query spaces for positive and negative interactions. \myNum{e} Finally, the proposed \attnname, $\cA_h^{Q^\pm V^\pm}$, outperforms all baselines as it can model token interactions better with separate value subspaces for positive and negative interactions.

\begin{table}[t]
\centering
\scriptsize
\begin{tabular}{lcccccc}
\toprule
\multirow{2}{*}{\textbf{Model}} & \multicolumn{2}{c}{\textbf{Validation Set}}       & \multicolumn{2}{c}{\textbf{Test Set}} & \multirow{2}{*}{\textbf{Parameters}} & \multirow{2}{*}{\textbf{GFLOPS}} \\
\cmidrule(lr){2-3}\cmidrule(lr){4-5}
                                                & \textbf{Acc.@1} & \textbf{Acc.@5} & \textbf{Acc.@1} & \textbf{Acc.@5}     &                &                \\
\midrule
ViT-B                                           & 81.1            & 95.6            & 81.1            & 95.6        & 86.6M         & 17.6          \\
ViT-B\textsuperscript{\textdagger}              & 78.8  (\reddown 2.3)         & 94.4    (\reddown 1.2)        & 78.7   (\reddown 2.4)         & 94.5      (\reddown 1.1)      & 100.7M     & 17.7                \\
ViT-B\textsuperscript{\ddag}              & 81.3   (\greenup {0.2})         & 95.5   (\reddown {0.1})       &       81.3   (\greenup {0.2})     &      95.6          & 100.7M     & 20.7             \\
\midrule
\attnname ($\cA_h^{\pm}$)                      & 78.8 (\reddown 2.3)  & 94.0 (\reddown 1.6)  & 78.7 (\reddown 2.4) & 94.0 (\reddown 1.6)           & 86.6M      & 17.6           \\
\attnname ($\hat{\cA}_h^{\pm}$)                & 81.1                   & 95.5 (\reddown  0.1) & 81.0 (\reddown 0.1) & 95.6                    & 86.6M          & 17.6          \\
    \attnname ($\cA_h^{Q^\pm V}$)                  & \underline{81.7} (\greenup \underline{0.6})  & \underline{95.8} (\greenup  \underline{0.2}) & \underline{81.7} (\greenup \underline{0.6}) & \textbf{95.8} (\greenup \textbf{0.2})      & 93.6M   &  19.3          \\
\attnname ($\cA_h^{Q V^\pm}$)                  & 81.1                   & 95.6                   & 81.2 (\greenup 0.1) & 95.7 (\greenup 0.1)    & 93.6M     & 19.3            \\
\rowcolor{Gray}
\textbf{\attnname ($\cA_h^{Q^\pm V^\pm}$)}     & \textbf{81.9 (\greenup  0.8)} & \textbf{95.9 (\greenup  0.3)} & \textbf{81.9 (\greenup  0.8)} & \textbf{95.8 (\greenup  0.2)} & 100.7M & 21.1  \\
\bottomrule
\end{tabular}
\caption{Performance comparison of accuracy, FLOPs, and model parameters on the ImageNet-1K~\cite{deng2009imnet} validation and test set for our intermediate baselines. The \greenup~and \reddown~arrows indicate the absolute accuracy gain and loss compared to the baseline ViT-B, respectively.}
\label{tab:ablation_qv}
\end{table}

\begin{figure}[t]
    \centering
    \includegraphics[width=0.75\linewidth]{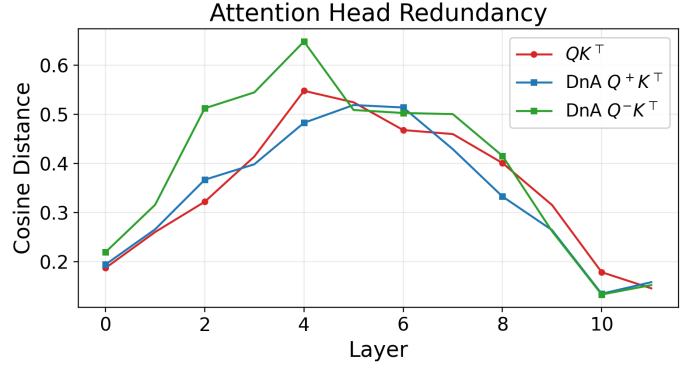}
    \caption{Mean head redundancy measured by pairwise cosine distance between attention weights across heads. Higher distance indicates lower redundancy.
    }
    \label{fig:head_redundancy}
\end{figure}

\myNum{B}\smartparagraph{Is the performance gain of our proposed model due to parameter increase?} Due to separate query and value, \attnname involves extra parameters; see \Cref{tab:ablation_qv}. The proposed \attnname, $\cA_h^{Q^\pm V^\pm}$, has 14M extra parameters than the ViT-B. Thus, a natural question is whether the superior performance of \attnname is attributed to the extra parameters. To address this, we trained two parameter-matched ViT-B variants. ViT-B\textsuperscript{\textdagger} has 14M extra parameters in the value branch; ViT-B\textsuperscript{\ddag} has 14M extra parameters across the attention dimension. ViT-B\textsuperscript{\textdagger} underperforms ViT-B by $2.4\%$; ViT-B\textsuperscript{\ddag} improves over ViT-B by only $0.2\%$; in contrast DnA improves over ViT-B by $0.8\%$. This indicates the performance gain is not due to extra parameters, but careful architectural and design choices~\cite{Zhai_2022_CVPR}. 

\begin{figure*}
    \centering
    \includegraphics[width=\linewidth]{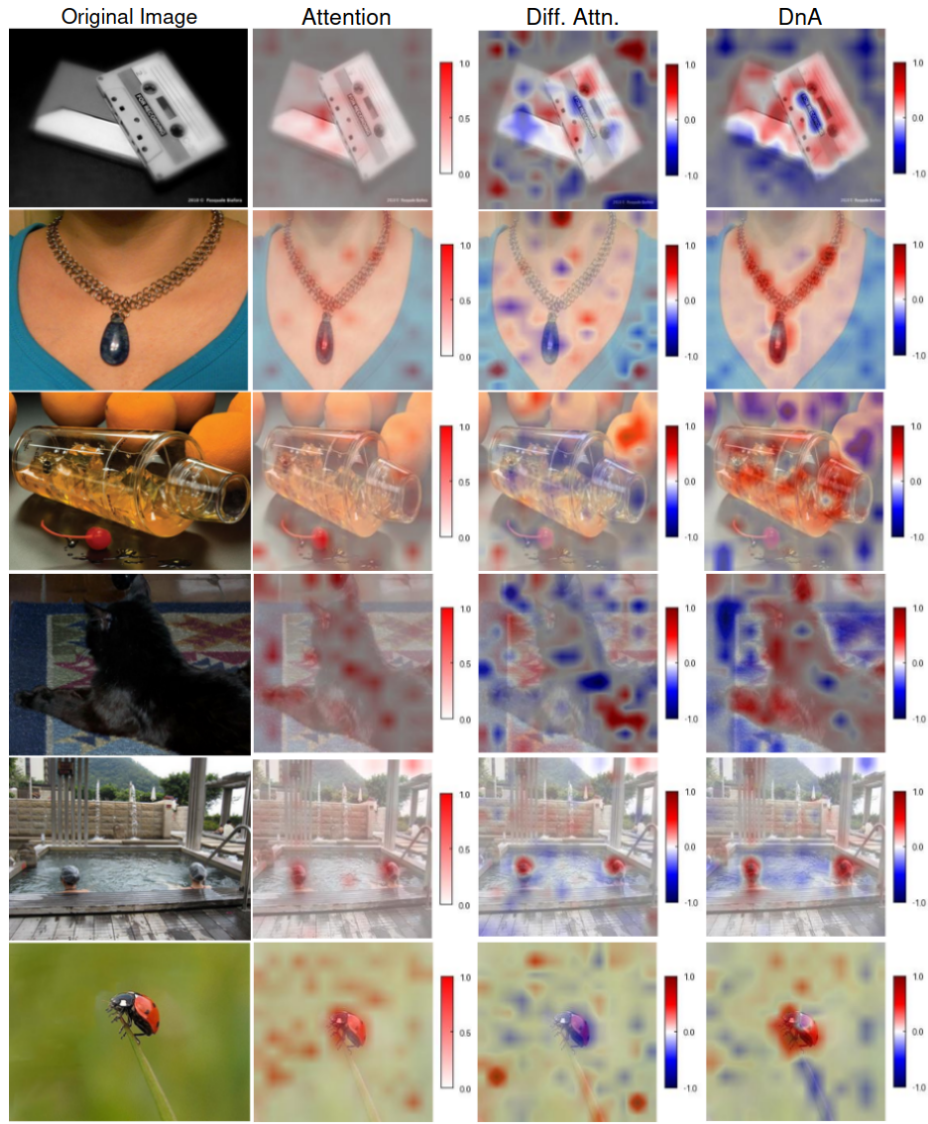}
    \caption{Attention visualization for ViT-B~\cite{vit} using softmax attention, differential attention~\cite{yedifferential}, and our \attnname. From top to bottom, the objects are: \textbf{cassette, chain link, cocktail shaker, Egyptian cat, swimming cap}, and \textbf{ladybug.}}
    \label{fig:additional_visualizations}
\end{figure*}

\myNum{C}\smartparagraph{Attention head similarity.}
To understand how attention heads specialize within each positive and negative branch, we measure head redundancy by computing, $\Delta_{ij}:=1-\cos(\angle \ \boldsymbol{\sigma}(Q_i^\pm K_i^\top),\boldsymbol{\sigma}( Q_j^\pm K_j^\top))$, between attention weight distributions, $\boldsymbol{\sigma}(Q_i^\pm K_i^\top),\boldsymbol{\sigma}(Q_j^\pm K_j^\top)$, and average across heads. Figure~\ref{fig:head_redundancy} outlines the differences in head specialization patterns. The negative branch displays substantially higher head diversity (lower redundancy) early in the network, peaking at layer 4 with a measurement of approximately $+0.10$ over the baseline. This suggests that different heads in the negative branch learn to identify distinct noise patterns. After layer 8, the negative branch becomes more redundant than the baseline, indicating that once noise suppression is complete, the heads become more similar in behavior. The positive branch maintains a trend closer to the baseline (if not less diverse) across its encoder layers.

\myNum{D}\smartparagraph{Distributions of $\alpha_h$ under different denoising designs.}
\Cref{fig:alphas} shows that the learnable parameter, $\alpha_h$ varies under different \attnname designs. The proposed \attnname, $A_h^{Q^\pm V^\pm}$, follows the smoothest progression, where we observe that the per-layer average value of $\alpha_h$ incrementally increases as the network deepens. We also see from the standard deviation plot that $A_h^{Q^\pm V^\pm}$ has the smallest per-head variance across models. This is not to say that there is redundancy across heads (which we can observe from \Cref{fig:head_redundancy}), rather that the per-layer magnitude of \textit{denoising} heads is similar. For $\cA_h^{Q^\pm V}$ and $\cA_h^{Q V^\pm}$, we observe a larger range between the values of $\alpha_h$, and a mean that exhibits less linear progression. Lastly, for $\hat{\cA}_h^{\pm}$, we see that for the first 8 layers, the minimum $\alpha_h$ actually takes on a negative value. This indicates that for some heads, the model actually chooses to \textit{add} from the denoising branch, rather than subtract it. We hypothesize that due to the constraint of having shared query and value spaces, this variant lacks the finer-grained denoising needed to distinguish closely related features. This points to a need for the flexibility provided by separate queries and values per branch.

\begin{figure*}[t]
    \centering
    \includegraphics[width=\linewidth]{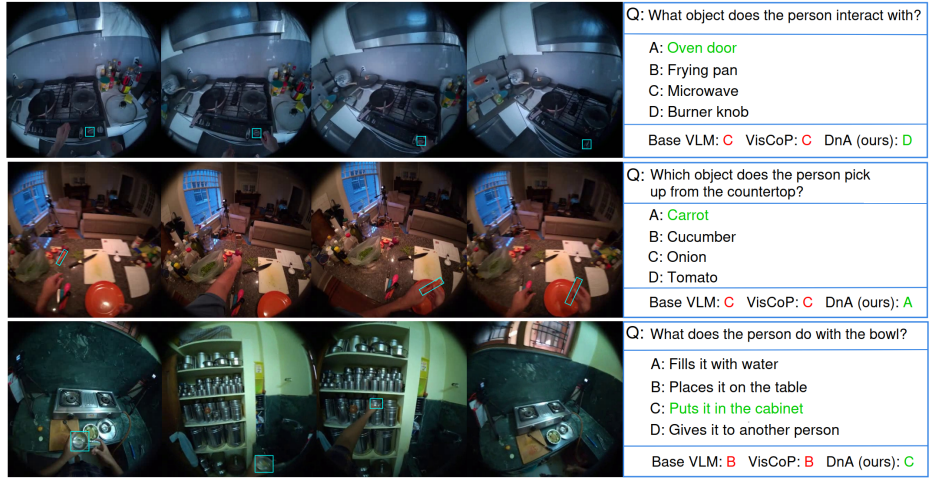}
    \caption{Examples of \attnname outperforming baseline methods on egocentric video QA.  For each example, frames are uniformly sampled, with the object of interest in a cyan box. The correct answer is in \textcolor{correctgreen}{\textbf{green}}, the incorrect in \textcolor{red}{\textbf{red}}.}
    \label{fig:additional_ego_vis}
\end{figure*}

\myNum{E}\smartparagraph{Visualizations.}
\myNum{i}\textit{Image Classification:} Chefer \etal~\cite{chefer2021transformer} calculates local relevance by taking the Hadamard product of ReLU-filtered attention gradients and attention maps, isolating positive contributions to the classification; see \S \ref{sec:analysis}. Layer-wise relevancy scores are then propagated through the transformer using recursive matrix multiplication, with an identity matrix added at each step to account for skip connections. This gradient-weighted method produces class-specific interpretations of attention flow by propagating relevance during its backpropagation in the network. \Cref{fig:additional_visualizations} shows how \attnname focuses attention on the object of interest. The negative branch (\textcolor{blue}{blue}) effectively denoises the background semantics, resulting in sharper attention from the positive branch (\textcolor{red}{red}). This results in clearer attention visualizations compared to ViT-B~\cite{vit} and differential attention~\cite{yedifferential}. \myNum{ii}\textit{Video QA:} In Figure~\ref{fig:additional_ego_vis} we observe three cases where the Base VLM and VisCoP select incorrect options, while \attnname correctly identifies the object. In all three cases, the baselines select visually reasonable but incorrect options, while \attnname attends more precisely to the query, filtering the irrelevant visual context.

\myNum{F}\smartparagraph{Entropy.}
\label{app:entropy} We analyze the entropy of attention distributions across encoder layers. Attention entropy, measured as the Shannon entropy, $H$ (see Theorem \ref{thm:maxent}), of the attention weights, quantifies the concentration of attention. Lower entropy indicates attention is focused on a few specific tokens, while a higher entropy indicates that attention is \textit{dense} across many tokens~\cite{hyeon2023scratching, ma2022close}. In ViTs, the desirability of high or low attention entropy is task-dependent. Outside of attention pruning contexts, dense attention is preferred, as it has been empirically linked with more stable training while promoting holistic token interactions~\cite{zhai2023stabilizing,hyeon2023scratching,ma2022close}.

We independently analyze both branches of \attnname ($\cA^{Q^{+}V^{+}}_h$ and $\cA^{Q^{-}V^{-}}_h$) to understand their distinct roles. \Cref{fig:entropy_depth} shows the entropy structure of attention weights across layers. By sampling the Top-$k\%$ entropy values per sample and contrasting them with the mean, we estimate how many tokens contribute to the attention density. In the negative \emph{denoiser} branch, higher entropy suggests that the denoising signal aggregates information from a larger set of tokens, and lower entropy indicates a more direct and confident filtering process. Overall, all models exhibit a U-shaped trend, meaning attention is initially dense, then it sharpens in early-to-middle layers, before becoming dense as the network deepens.

\begin{figure}[t]
    \centering
    \includegraphics[width=1\columnwidth]{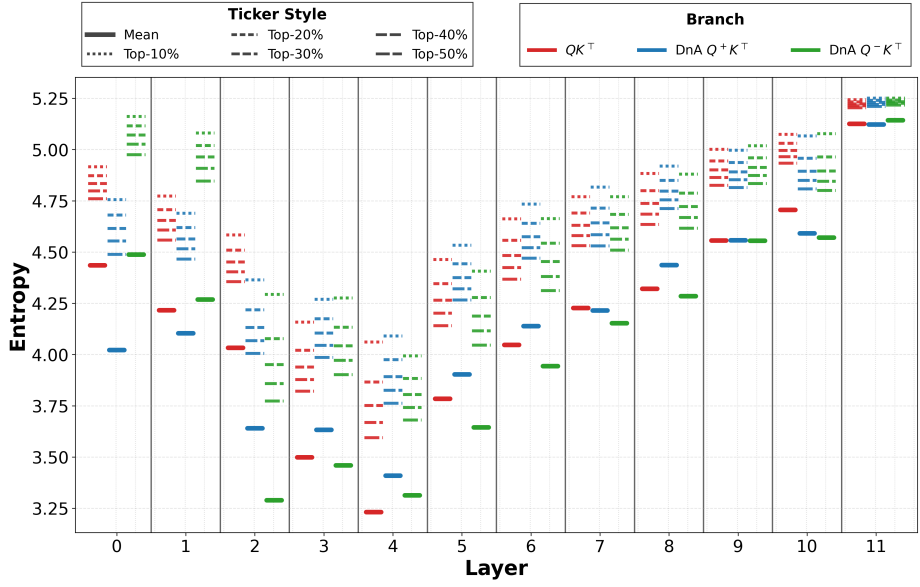}
    %\vspace{-7mm}
    \caption{\small{Layer-wise entropy of attention weights for ViT-B using softmax attention (\textcolor{vitred}{$QK^\top$}) and both \attnname attention branches (\textcolor{posblue}{$Q^+K^\top$} and \textcolor{neggreen}{$Q^-K^\top$}) on the ImageNet-1K validation set. Solid lines show mean entropy across all attention heads. Dashed lines represent the entropy of the Top-$k$\% most attended tokens.}
    }
    \label{fig:entropy_depth}
\end{figure}
\begin{table}[H]
\centering
\begin{tabular}{l|c|cccc|c}
\toprule
\textbf{Model} & \textbf{Scaling Factor} & \textbf{Act.} & \textbf{Task} & \textbf{HOI} & \textbf{Hand} & \textbf{Avg.} \\
\midrule
VideoLLaMA3 & N/A & 74.9 & 75.8 & 75.2 & \textbf{65.3} & 72.8 \\
VisCoP & N/A & 81.8 & 86.1 & \textbf{79.3} & 65.1 & 78.1 \\
\rowcolor{Gray}
\textbf{+DnA} & $10^{-2}$ & 83.0 & 86.5 & \textbf{79.3} & 65.1 & 78.5 \\
\rowcolor{Gray}
\textbf{+DnA} & $10^{-4}$ & \textbf{83.1} & \textbf{87.1} & 79.2 & 65.1 & \textbf{78.6} \\
\bottomrule
\end{tabular}%
\caption{Effect of negative branch scale on \attnname for Video LLMs.}
\label{tab:video_llm_scale}
\end{table}

\myNum{G}\smartparagraph{Initialization of Negative Branch in Video LLM.}
\label{app:branch_init}
We discuss the need for initializing the negative branch as a scaled-down copy of the positive branch when incorporating \attnname within video LLM training \S\ref{sec:video_llm}. Because we are fine-tuning a pre-trained VLM, we require both the positive and negative branches to inherit prior knowledge. If we initialize both branches of \attnname with identical weights, the two cancel out. In Table~\ref{tab:video_llm_scale}, the performance is mostly unchanged across scaling factors, showing that the exact value is not critical for model convergence once this weight symmetry is broken. We also note that when training from scratch, the randomness of the initialization prevents the weights of both branches from being identical, consequently avoiding this weight symmetry.

\end{document}